\documentclass[twoside]{article}
\usepackage[accepted]{aistats2022}
\usepackage[round]{natbib}

\usepackage[utf8]{inputenc} % allow utf-8 input
\usepackage[T1]{fontenc}    % use 8-bit T1 fonts
\usepackage{hyperref}       % hyperlinks
\usepackage{url}            % simple URL typesetting
\hypersetup{
    colorlinks=true,
    linkcolor=blue,
    citecolor=blue,      
    urlcolor=brown,
}
\usepackage{booktabs}       % professional-quality tables
\usepackage{amsfonts}       % blackboard math symbols
\usepackage{nicefrac}       % compact symbols for 1/2, etc.
\usepackage{microtype}      % microtypography
\usepackage{xcolor}         % colors
\usepackage{amsmath}
\usepackage{amssymb}
\usepackage{amsthm}
\usepackage{babel}
\usepackage{adjustbox}
\usepackage{dsfont}
\usepackage{wrapfig}
\usepackage{graphicx,subfigure}
\usepackage{algorithm}
\usepackage{algpseudocode}
\usepackage{mathtools}
\usepackage{enumitem}
\setlist[itemize]{leftmargin=5mm,itemsep=0.5mm}
\setlist[enumerate]{leftmargin=*,itemsep=0.5mm}
\newcommand{\calX}{\mathcal{X}}
\newcommand{\calH}{\mathcal{H}}

\newcommand{\calF}{\mathcal{F}}
\newcommand{\calA}{\mathcal{A}}

\newcommand{\calN}{\mathcal{N}}
\newcommand{\calO}{\mathcal{O}}
\newcommand{\bbR}{\mathbb{R}}

\newcommand{\bbE}{\mathbb{E}}

\definecolor{aoe}{rgb}{0.0, 0.5, 0.0}
\definecolor{amber}{rgb}{1.0, 0.49, 0.0}

\DeclareMathOperator{\Tr}{Tr}
\newcommand{\riemannGrad}{\text{grad}}

\newenvironment{talign*}
 {\csname align*\endcsname}
 {\endalign}
\newenvironment{talign}
{\align}
{\endalign}

\newtheorem{assumption}{Assumption}
\newtheorem{remark}{Remark}
\newtheorem{lemma}{Lemma}
\newtheorem{theorem}[lemma]{Theorem}
\newtheorem{proposition}[lemma]{Proposition}
\newcommand\autowidehat[1]{%
\savestack{\tmpbox}{\stretchto{%
  \scaleto{%
    \scalerel*[\widthof{\ensuremath{#1}}]{\kern0pt\bigwedge\kern0pt}%
    {\rule[-\textheight/2]{1ex}{\textheight}}%WIDTH-LIMITED BIG WEDGE
  }{\textheight}% 
}{0.5ex}}%
\stackon[1pt]{#1}{\tmpbox}%
}

\begin{document}

% If your paper is accepted and the title of your paper is very long,
% the style will print as headings an error message. Use the following
% command to supply a shorter title of your paper so that it can be
% used as headings.
%
%\runningtitle{I use this title instead because the last one was very long}

% If your paper is accepted and the number of authors is large, the
% style will print as headings an error message. Use the following
% command to supply a shorter version of the authors names so that
% they can be used as headings (for example, use only the surnames)
%
\runningauthor{Xing Liu, Harrison Zhu, Jean-Fran\c{c}ois Ton, George Wynne, Andrew Duncan}

\twocolumn[

\aistatstitle{Grassmann Stein Variational Gradient Descent}

\aistatsauthor{ 
    Xing Liu \\ 
    Imperial College London \\
    \And 
    Harrison Zhu \\
    Imperial College London \\
    \And 
    Jean-Fran\c{c}ois Ton \\ 
    University of Oxford \\
    \AND 
    George Wynne \\
    Imperial College London \\
    \And 
    Andrew Duncan \\
    Imperial College London}

% \aistatsaddress{ Imperial College London \And  Imperial College London \And University of Oxford \\ \AND Imperial College London \And Imperial College London } ]

\aistatsaddress{ } ]

\begin{abstract}
Stein variational gradient descent (SVGD) is a deterministic particle inference algorithm that provides an efficient alternative to Markov chain Monte Carlo. However, SVGD has been found to suffer from variance underestimation when the dimensionality of the target distribution is high. Recent developments have advocated projecting both the score function and the data onto real lines to sidestep this issue, although this can severely overestimate the epistemic (model) uncertainty. In this work, we propose \textit{Grassmann Stein variational gradient descent} (GSVGD) as an alternative approach, which permits projections onto arbitrary dimensional subspaces. Compared with other variants of SVGD that rely on dimensionality reduction, GSVGD updates the projectors simultaneously for the score function and the data, and the optimal projectors are determined through a coupled Grassmann-valued diffusion process which explores favourable subspaces. Both our theoretical and experimental results suggest that GSVGD enjoys efficient state-space exploration in high-dimensional problems that have an intrinsic low-dimensional structure.
\end{abstract}

\section{INTRODUCTION}
\label{sec:introduction}
Variational inference (VI) \citep{blei2017variational} is an optimisation-centric framework for approximating complex distributions that are intractable (only known up to a scale factor): given any target distribution $p(x)$, VI searches over a user-defined class of distributions $\mathcal{Q}$ for an optimal $q(x) \in \mathcal{Q}$ that is closest (in terms of a discrepancy or divergence) to $p(x)$. Extensively applied in the field of Bayesian inference, VI has the attraction of being scalable to big datasets, although it leads to biased estimation unless $\mathcal{Q}$ is broad enough to include the target distribution. This is in contrast to Markov chain Monte Carlo (MCMC) algorithms \citep{gilks1995markov}, which allows asymptotically exact sampling from the true target distribution, but does not scale well to big datasets and high dimensionality due to long mixing times \citep{levin2017markov}.

% \jef{What is the factor here O(..,)}\hz{it's the mixing time, albeit linear to the number of datapoints and dim  I SEE}
%(although scalable MCMC algorithms have been recently investigated. See \citep{welling2011bayesian}). 

The choice of $\mathcal{Q}$ in VI is crucial to guarantee a good approximation to $p(x)$ while retaining computational tractability. The classical mean-field approximation, which assumes that the distributions in $\mathcal{Q}$ have independent marginals, can be overly simplistic in many cases. To address this, a growing line of work in VI with Normalising Flows (NFs) \citep{rezende2015variational} seeks to construct an invertible map $T$ so that the pushforward distribution of $T(x)$, with $x\sim q$ and $q \in \mathcal{Q}$, will form a flexible approximation to $p$. 

Stein variational gradient descent (SVGD) \citep{Liu2016} was introduced as a particle-based variational inference method in which the map $T$ seeks to move a set of {particles} along a vector field which is chosen within a reproducing kernel Hilbert space (RKHS) to optimally transport towards the target $p(x)$.  Building on the kernel Stein discrepancy \citep{Liu2016}, the SVGD transport both drives the particles to the high-probability regions of $p(x)$ and enforces repulsion between the particles to prevent mode-collapse. It hence has the advantage of being ``particle-efficient'' in that a small number of particles can achieve good approximation of $p(x)$. Although this has made SVGD a popular tool in a range of applications including meta-learning \citep{yoon2018bayesian} and learning diversified mixture models \citep{wang2019nonlinear}, it is found that the marginal variance of the resulting particles scales inversely with the dimension, resulting in under-estimation of the variance. This was studied analytically in \cite{ba2022understanding}, in which the authors attributes this issue to \emph{(i)} an uneven scale of variance for the two terms in the SVGD update, and \emph{(ii)} a deterministic bias due to the dependence of the particle positions at each step.

Recently, \cite{gong2021sliced} proposed a  \emph{sliced} version of kernel Stein discrepancy (KSD) analogous to other sliced discrepancies \citep{kolouri2019generalized}, and an associated \emph{sliced SVGD} (S-SVGD). This variant of SVGD decomposes the kernel-based dynamics along a sequence of 1-dimensional projections, called \emph{slices}.  Empirical results \cite[Appendix J.5, Table 5]{gong2021sliced} showed that the S-SVGD dynamic can mitigate the variance under-estimation issue of SVGD in high dimensions. However, S-SVGD is constrained to 1-dimensional projection, and only seeks the optimal slices for the data but keeps the slices for the score function deterministic. As we demonstrate in this paper, these can lead to inflated variance of the S-SVGD estimation. 
% On the contrary, SVGD will underestimate the uncertainty.

In this work, we introduce a unified approach where a non-uniform probability distribution over the projectors for \emph{both the score function and the data} are adaptively updated to emphasise directions in which there is largest discrepancy. This is achieved by introducing a stochastic diffusion process taking values in the Grassmann manifold \citep{bendokat2020grassmann} which evolves along with SVGD particles. By tuning the diffusion process parameters we can adjust the trade-off between exploration and exploitation of suitable projections. In addition, using projections onto higher dimensional spaces allows us to take into account the correlations/interactions between components and thus producing more accurate uncertainty estimates.

% \ad{TODO: Update the story.  Broadly speaking, the introduction of slicing into a discrepancy takes one of two forms: the \emph{average case sliced discrepancy } where we introduce a probability measure over the space of projectors and compute the average projected discrepancy, or the \emph{worst case sliced discrepancy} where one considers the maximum projected discrepancy over all projectors. Both approaches offer distinct advantages and present a trade-off between exploration and computational efficiency.  In this work we introduce a unified approach where a non-uniform probability distrbution over the slices is adaptively updated to emphasise directions in which there is largest discrepancy.  This is achieved by introducing a stochastic diffusion process taking values in the Grassmann manifold which evolves along with SVGD particles.  By tuning the diffusion process parameters we can adjust the trade-off between exploration and exploitation of suitable slice directions.} 

% \textbf{Our contributions:} Motivated by S-SVGD, we derive an alternative SVGD algorithm that also uses the idea of low-dimensional projections, which we call Grassmann SVGD (GSVGD), by projecting onto $m\geq1$ dimensions using the Grassmann manifold. By using a coupled Grassmann-valued diffusion process, we can learn the more efficiently explore the projections whilst performing SVGD. We show numerically that our method is competitive to SVGD and S-SVGD on high-dimensional synthetic and benchmark problems.
\textbf{Contributions:} 
Motivated by S-SVGD, we propose a novel algorithm, Grassmann SVGD (GSVGD), which employs $m$-dimensional projectors (compared with the $1$-dimensional ones in S-SVGD) to evolve the particles towards the target distribution and to mitigate variance under-estimation. We show numerically that our method is competitive to SVGD and S-SVGD on high-dimensional synthetic and benchmark problems while more accurately estimating the epistemic uncertainty.

\begin{figure}
    \centering
    \includegraphics[width=1\linewidth]{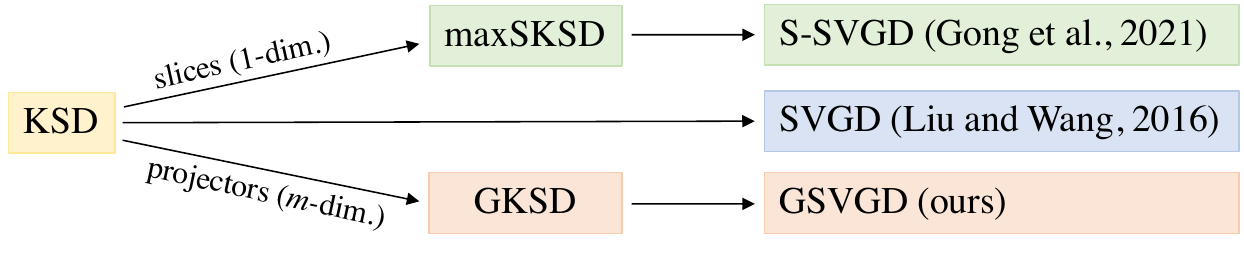}
    \vspace{-0.3cm}
    \caption{A summary of different SVGD algorithms using no projections (SVGD), 1-dimensional projections (S-SVGD), and $m$-dimensional projections (GSVGD) with an arbitrary $1 \leq m \leq d$, where $d$ is the dimensionality of the problem. }
    \label{fig: summary}
\end{figure}

\section{BACKGROUND}
\label{sec:background}
\paragraph{Stein Variational Gradient Descent:} Let $\calX = \bbR^{d}$ and $P$ be a probability measure over $\calX$ with smooth positive density $p$ which we can evaluate up to a normalisation constant. We are interested in approximating $P$ by transporting a known measure $Q$, defined over $\calX$ and with smooth density $q$, to $P$ via a sequence of maps that minimise a given loss. 

Choosing the loss to be the Kullback-Leibler (KL) divergence defined as $\text{KL}(Q,P) = \int_{\calX}q(x)\log(q(x)/p(x))dx$ allows us to minimise the discrepancy $\text{KL}(T_{\#}Q,P)$  over maps $T\colon\calX\rightarrow\calX$ where $T_{\#}Q$ is the pushforward of $Q$ with respect to $T$. In \citet{Liu2016SVGD}, the authors choose a specific parametrisation $T(x) = x+\varepsilon\phi(x)$ where $\phi$ lies within the unit ball $\mathcal{B}_k^d:=\{\phi\in\mathcal{H}_k^d: ||\phi||_{\mathcal{H}_k^d}\leq 1\}$ of the Cartesian product $\mathcal{H}_k^d:=\bigtimes_{i=1}^d \mathcal{H}_k$ of the RKHS $\mathcal{H}_k$ associated with kernel $k:\mathcal{X}\times\mathcal{X}\rightarrow\mathbb{R}$. The crucial observation of \citet{Liu2016SVGD} is that the maximal rate of decay of this discrepancy with respect to $\epsilon$ is the kernel Stein discrepancy (KSD) \citep{Chwialkowski16,Liu2016} given by
\begin{talign}
    \text{KSD}(Q,P) 
    &= \sup_{\phi\in\mathcal{B}_k^d} \mathbb{E}_Q[\calA_{p}\phi(x)],\label{eq:KSD_def}
\end{talign}
where $x\sim Q$,  $s_{p}(x) = \nabla_{x}\log p(x)$ is the \emph{score function} and $\mathcal{A}_{p} \phi(x) = s_{p}(x)^\intercal\phi(x) + \nabla \cdot \phi(x)$ is the \textit{Stein operator} associated to $P$. The key advantage of KSD is that it quantifies the difference between two measures $P$ and $Q$ where only the score function of $P$ is available, which is the typical situation in Bayesian inference for complex models.  The supremum in \eqref{eq:KSD_def} is attained by the function $\phi^{*}(\cdot) 
\coloneqq \mathbb{E}_{Q}[\mathcal{A}_p k(\cdot,x)] =\mathbb{E}_Q[k(\cdot, x)s_{p}(x) +   \nabla_{x} k(\cdot, x)]$, which leads to the remarkable result \citep[Theorem 3.3]{Liu2016SVGD} $\nabla_\phi \text{KL}(T_{\#}Q,P)|_{\phi=0} = -\phi^{*}(x)$ where $\nabla_{\phi}$ is the functional derivative and $T(x)$ is the map parameterised as above. This suggests that $\phi^{*}$ is, according to KL divergence, the best choice of function to use in $T$ to map $Q$ to $P$. 

The expression for the optimal $\phi^*$ forms the basis of the SVGD algorithm.  Particles $X^t = (x_1^t,\ldots, x_N^t)$ at step $t$, are evolved using the following update rule
\begin{talign}
 \label{eq:svgd_vanilla}
    x_i^{t+1} 
    &= x_i^t + \epsilon  \widehat{\phi}^\ast(x_i^t),
\end{talign}
for $i=1,\ldots, N$ and $t=0, 1,2,\ldots$, where \begin{talign}
    \label{eq: svgd_vanilla update}
    \widehat{\phi}^*(\cdot) = \frac{1}{N}\sum_{j=1}^{N}\big[k(\cdot,  x_j^t)\nabla_{x^t_j}\log p(x_j^t) + \nabla_{x_j^t} k(\cdot, x_j^t)\big]
\end{talign} is an estimator of $\phi^*(\cdot)$ obtained by replacing $Q$  with  $\frac{1}{N}\sum_{i=1}^N \delta_{x_i^t}$. Intuitively, \eqref{eq:svgd_vanilla} is an Euler discretisation with step-size $\epsilon$ of the following system of ODEs
\begin{talign*}
    \frac{d x_i}{dt}(t) 
    &= \frac{1}{N}\sum_{j=1}^{N}\big[k(x_i(t),  x_j(t))\nabla_{x_j(t)}\log p(x_j(t)) \\ 
    &\quad + \nabla_{x_j(t)} k(x_i(t), x_j(t))\big], \quad i=1,\ldots,N,
\end{talign*}
for particle positions $x_1(t),\ldots, x_N(t)$.

\paragraph{Tackling the Curse of Dimensionality:}
The SVGD algorithm has shown to be asymptotically exact, in the sense that in the limit of infinite particles and vanishing step-size the particles will converge to the target density \citep{lu2019scaling}.  However, in  the finite particle regime,  SVGD may fail to adequately explore the state-space, exhibiting mode collapse and variance under-estimation, particularly in high dimensions \citep{Liu2016,zhuo2018message,gong2021sliced}.

The deterioration in performance in high-dimensions can be understood from the update rule defined in \eqref{eq: svgd_vanilla update}. The RHS consists of two terms: a kernel-averaged score function $\frac{1}{N}\sum_j k(\cdot, x_j^t)\nabla\log p(x_j^t)$ which attracts particles towards modes of the target density, and a repulsion term $\frac{1}{N}\sum_j \nabla k(\cdot, x_j^t)$ which encourages particle diversity by pushing nearby particles away from each other. As has been observed in \citet[Section 3.2]{Liu2016SVGD} and studied analytically in \cite[Section 3]{zhuo2018message}, the influence of the repulsion term drops dramatically with increasing dimension, effectively reducing SVGD to a gradient ascent method for $\log p$. As a result, the estimation of SVGD can suffer from a significant under-estimated variance for high-dimensional targets. 

Various developments have been proposed to mitigate this issue; we now discuss the most relevant to our work, the sliced approach of \citet{gong2021sliced}, and defer a review of the others to Section \ref{sec:related_work}. \citet{gong2021sliced} addressed this problem by projecting both the particles and the score function along one-dimensional directions, known as \emph{slices}, leading to the \emph{max sliced kernel Stein discrepancy} (maxSKSD). This approach hinges on the fact that measuring KSD along finitely many slices is sufficient to capture all geometric information of two distributions, as long as optimal slices are used. Given a user-defined orthonormal basis $O$ in $\bbR^d$, the \emph{maxSKSD} takes the form
\begin{align}
    &\text{maxSKSD}(Q,P) 
     =
     \sum_{r \in O} \sup_{g_{r}\in\mathbb{S}^{d-1}}\sup_{\phi \in \mathcal{B}_{k}} \bbE_Q[ s_p^r(x) \phi(x^{\intercal}g_{r}) \nonumber \\ 
     &\quad + r^\intercal g_{r} \nabla_{x^{\intercal}g_{r}} \phi(x^{\intercal}g_{r})],\label{eq:maxSKSD}\\
    &= \sum_{r \in O}\sup_{g_{r}\in\mathbb{S}^{d-1}}\sup_{\phi \in \mathcal{B}_{k_{r,g_{r}}}}\bbE_{Q}[s_{p}(x)^\intercal\phi(x) + \nabla\cdot\phi(x)]\label{eq:maxSKSD_mat}
\end{align}
where $s^r_p(x) \coloneqq s_p(x)^\intercal r$ is the score projected along a slicing direction $r$ and $\mathcal{B}_{k}$ is the unit ball of from RKHS $\calH_{k}$ for some kernel $k\colon\bbR\times\bbR\rightarrow\bbR$. The subscript $g_{r}$ is used to emphasise that for each $r$ a new optimal $g_{r}$ must be found. In \eqref{eq:maxSKSD_mat}, $\mathcal{B}_{k_{r,g_{r}}}$ is the unit ball of the RKHS corresponding to the matrix valued kernel $k_{r,g_{r}}(x,y) = rr^{\intercal}k(g_{r}^{\intercal}x,g_{r}^{\intercal}y)$. This can be seen from standard results regarding matrix valued kernels \citep[Chapter 6]{Paulsen2016} and this will be relevant to our later developments.

The idea of the maxSKSD approach is that the $r\in O$ will slice the score function of the target, and for each of those slices the input is sliced according to some other directions $g_{r}$. It is known that, under some standard regularity conditions, maxSKSD discriminates measures \citep[Corollary 3.1]{gong2021sliced}. Replacing KSD with maxSKSD in SVGD results in \emph{sliced-SVGD} (S-SVGD; \citet{gong2021sliced}). 

An issue with this procedure is that the basis $O$ is a user-defined choice and may result in inefficient performance. Additionally, restricting $g_{r}$ and $r$ to one dimension might not capture the underlying structure as effectively as higher dimensional projections. Our approach, described in Section \ref{sec:GSVGD}, will use projections onto subspaces that have dimension potentially greater than one. The framework of these subspace projections is described next. This framework also allows us to explore projections in the underlying geometry more efficiently by using Riemannian optimisation.

\paragraph{Grassmann Manifold}
Let $m \in \{ 1, 2, \ldots, d\}$. A crucial ingredient of our proposed method is the \emph{Grassmann manifold} of $m$-dimensional subspaces in $\bbR^d$ \citep{bendokat2020grassmann}, defined as
\begin{talign*}
    \textrm{Gr}(d, m) = \{ E \subseteq \bbR^d: \textrm{dim}(E) = m \}.
\end{talign*}
The intuition behind our proposed method is to define a KSD that seeks the worst possible discrepancy over all distinct subspaces of dimension $m$, where the worst-case subspace can be found with gradient-type optimisation on the Grassmann manifold. We briefly review the key ingredients of optimisation on the Grassmann manifold; a more detailed discussion is deferred to Appendix \ref{appendix: grassmann manifold}.

To make sense of optimisation on the Grassmann manifold, we first need to represent each subspace $E \in \textrm{Gr}(d, m)$ in memory. One way of doing so is by a projection operator that maps every $x \in \bbR^d$ to an element in the subspace. To this end, we define a \emph{projector} of rank $m$ to be a $d \times m$ matrix $A$ with orthonormal columns, i.e.\ $A^\intercal A = I_m$, where $I_{m}$ is the $m\times m$ identity matrix. It then follows that $A A^\intercal$ is a projection matrix since $(A A^\intercal)^2 = A A^\intercal$ and $(A A^\intercal)^\intercal = A A^\intercal$. Denote by $[A] = \{ A y: y \in \bbR^m \} \subset \bbR^d$ the image of $A$.

We represent each subspace $E \in \textrm{Gr}(d, m)$ by \emph{any} projector $A$ for which $[A] = E$. This is always possible since, for each projector $A$ the subspace $[A]$ is trivially an element of $\textrm{Gr}(d, m)$, and, conversely, for any $E \in \textrm{Gr}(d, m)$, we can construct a corresponding projector $A$ by column-wise appending the elements of any orthonormal basis of $E$. Such $A$ is unique up to an orthogonal transformation: for any projectors $A, B$, the subspaces $[A] = [B]$ if and only if $A = BC$ for some orthogonal matrix $C$ in $\bbR^{m \times m}$. As we will show in Section~ \ref{sec:GSD}, the proposed KSD does not depend on which projector is chosen so long as it corresponds to the same subspace $E$.

Given a representative $A$ of $[A] \in \textrm{Gr}(d, m)$, the \emph{tangent space} at $[A]$ is defined as $\mathcal{T}_{[A]}\text{Gr}(d,m)=\lbrace \Delta \in \mathbb{R}^{d \times m} : A^\intercal \Delta = 0 \rbrace$. We endow a Riemannian metric on $\mathcal{T}_{[A]}\text{Gr}(d,m)$ by restricting the standard matrix inner product $\langle \Delta, \tilde{\Delta} \rangle_0 = \Tr(\Delta^\intercal \tilde{\Delta})$ to elements $\Delta, \tilde{\Delta}$ in $\mathcal{T}_{[A]} \textrm{Gr}(d, m)$. It is then clear that $\Pi_A = I_d - A A^\intercal$ is a projection operator of $\bbR^{d \times m}$ onto $\mathcal{T}_{[A]} \textrm{Gr}(d, m)$: $\Pi_A G = (I_d - A A^\intercal) G \in \mathcal{T}_{[A]} \textrm{Gr}(d, m)$, for any $G \in \bbR^{d \times m}$ \cite[Proposition 3.53]{boumal2020introduction}.

\begin{figure}
    \centering
    \includegraphics[width=0.8\linewidth]{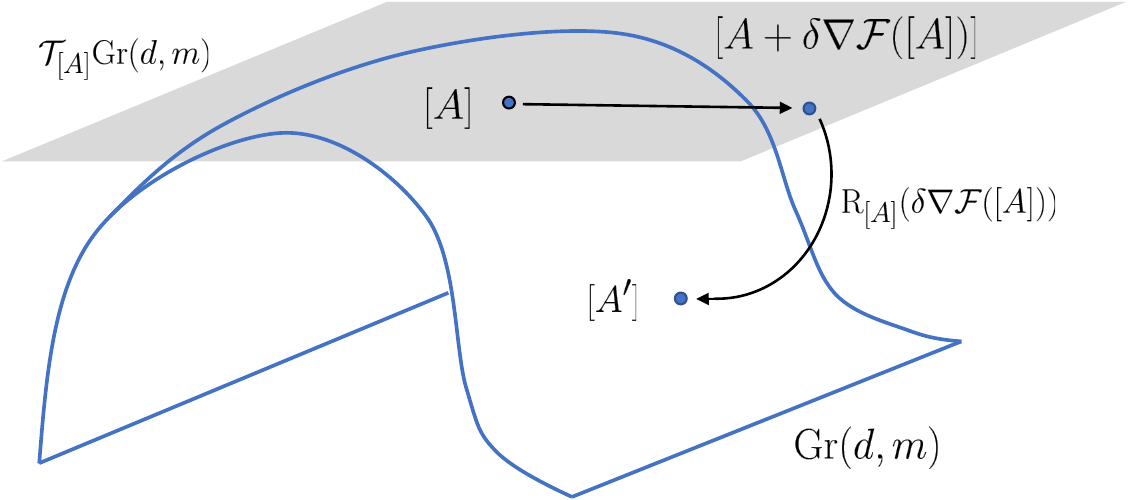}
    \caption{An illustration of a gradient ascent step on the Grassmann manifold.}
    \label{fig: riemannian gradient descent}
\end{figure}

Let $\calF$ be a function defined on $\text{Gr}(d,m)$ that is smooth in a proper sense \cite[Section 3]{boumal2020introduction}. Given a base point $[A] \in \text{Gr}(d,m)$, maximising $\calF$ requires \emph{(i)} finding a vector field of steepest ascent of $\calF$ at $[A]$, and \emph{(ii)} moving from $[A]$ along that vector field without leaving the manifold. 

On the Grassmann manifold, the direction of steepest ascent is no longer represented by the standard (Euclidean) gradient $\nabla \calF([A])$. Instead, it is given by the \emph{Riemannian gradient} of $\calF$ at $[A]$, given by
\begin{talign*}
    \riemannGrad \calF = \Pi_A \nabla \calF([A]), 
\end{talign*}
where $[\nabla \calF([A])]_{ij} = \partial \calF([A]) / \partial A_{ij} $. That is, $\riemannGrad \calF$ is the projection of the Euclidean gradient $\nabla \calF([A])$ onto the tangent space at $[A]$. For \emph{(ii)}, one means of moving along a manifold in a given direction is by the \emph{exponential map} \cite[Section 10.2]{boumal2020introduction}. In our work, we use the \emph{polar retraction} \citep[Definition 3.41]{boumal2020introduction}, which is an accurate (second order) approximation of the exponential map that is amendable to computation, and is defined as
\begin{talign}
    \label{eq: polar retr}
    \textrm{R}_{[A]}(\Delta ) = [UV^\intercal] ,
\end{talign}
where $[A] \in \text{Gr}(d,m)$, and $A + \Delta = USV^\intercal$ is a thin Singular Value Decomposition (see e.g.\ \citet[Section 9.6]{boumal2020introduction}). 

To summarise, with a step-size $\delta > 0$, one step to maximise $\calF$ starting from a base point $[A]$ is given by
\begin{talign*}
    [A'] = \textrm{R}_{[A]}(\delta \riemannGrad \calF).
\end{talign*}
See~Fig.\ref{fig: riemannian gradient descent} for an illustration. This is a generalisation of the standard gradient ascent to the Grassmann manifold \cite[Section 4.3]{boumal2020introduction}.

\section{GRASSMANN KERNEL STEIN DISCREPANCY}
\label{sec:GSD}
The starting point of our proposed approach is to introduce a novel Stein discrepancy which incorporates dimension reduction features through the use of a projector. To this end, for a projector $A$ of rank $m$ we introduce a matrix-valued kernel $k_A$ arising from an embedding on $\mathbb{R}^d$ into a $m$-dimensional subspace where $m\leq d$
\begin{talign*}
    k_A(x,y) = A A^\intercal k(A^\intercal x, A^\intercal y),\quad  x,y\in \mathbb{R}^d,
\end{talign*}
where $k:\mathbb{R}^m\times \mathbb{R}^m\rightarrow \mathbb{R}$ is a scalar-valued positive definite radial kernel on $\mathbb{R}^m$. That is, 
$k(u,v)=\Psi(\lVert u-v\rVert_{2}), u,v\in \mathbb{R}^m$, where $\Psi\colon\bbR\rightarrow\bbR$ and $\lVert\cdot\rVert_{2}$ is the Euclidean $2$-norm. Before proceeding, since we shall be performing optimisation over $\text{Gr}(d,m)$, we must be sure that for any two projectors $A,B$ that are equivalent in Gr$(d,m)$ the kernels $k_{A},k_{B}$ are equal. The next lemma assures us of this (see Appendix~\ref{appendix: proof of lemma grassmann_kernel} for the proof).

% \begin{lemma}
% \label{lemma: grassmann_kernel}
% % The Grassman Kernel Stein Discrepancy $KSD_{A}(P,Q)$ is independent of the particular representer of the Grassmann element $[A]\in \mbox{Gr}(d,m)$.
% Let $A,B\in\emph{St}(d,m)$ with $A\sim B$ in $\emph{Gr}(d,m)$ then $k_{A} = k_{B}$.
% \end{lemma}
\begin{lemma}
\label{lemma: grassmann_kernel}
Let $A,B$ be projectors of rank $m$ with $[A] = [B]$, then $k_{A} = k_{B}$.
\end{lemma}
Safe in the knowledge that $k_{A}$ is well defined over $\text{Gr}(d,m)$, we may define an associated KSD, which reveals how this choice of $k_{A}$ induces a projection of the data and score function. 
\begin{align}
    \mbox{KSD}_{A}(Q,P) 
    &= \sup_{\phi\in\mathcal{B}_{k_{A}}}\bbE_{Q}[s_{p}(x)\cdot \phi(x) + \nabla\cdot\phi(x)]
    \label{eq:GKSD_vec}\\
    = \sup_{\phi\in\mathcal{B}_{k}^{m}}\bbE_{Q}&\big[(A^\intercal s_{p}(x))\cdot\phi(A^\intercal x)
     + \nabla\cdot \phi(A^\intercal x)\big] .
    \label{eq:KSD_A}
\end{align}
where $\mathcal{B}_{k_{A}}$ is the unit ball of the RKHS corresponding to $k_{A}$ and $\mathcal{B}_{k}^{m}$ is the unit ball of the $m$-times Cartesian product of $\calH_{k}$, the RKHS corresponding to $k$. Contrasting this to \eqref{eq:maxSKSD}, setting $m=1$ and $r,g_{r} = A$ results in \eqref{eq:KSD_A}. The $r^{\intercal}g_{r}$ term in \eqref{eq:maxSKSD} becomes $I_{m}$ in \eqref{eq:KSD_A} since $A$ has orthonormal columns.

Analogous to the approach for KSD (see \citet{Liu2016}), we exploit the kernel trick to express $\text{KSD}_A$:
\begin{talign}
    &\text{KSD}_{A}(Q,P) \nonumber \\
    &=\mathbb{E}_{x, x'\sim Q}\left[ (A^\intercal s_{p}(x)) \cdot A^\intercal s_{p}(x')k(A^\intercal x,A^\intercal x')\right] \nonumber \\
    &\quad +2\mathbb{E}_{x, x' \sim Q}[ (A^\intercal s_{p}(x)) \cdot \nabla_{x_2}k(A^\intercal x,A^\intercal x')] \nonumber\\
    &\quad +\mathbb{E}_{x, x' \sim Q}[\Tr(\nabla_{x_1,x_2}k(A^\intercal x,A^\intercal x'))],\label{eq:KSD_double_integral}
\end{talign}
where $x, x'$ are i.i.d.\ random variables drawn from $Q$, $\nabla_{x_i} k(A^\intercal x, A^\intercal x')$ denotes the gradient of $k$ with respect to the $i$-th argument, $\nabla_{x_1,x_2}k$ is the matrix with $ij$-th entry $\partial^{2}k/\partial x_{i}\partial x_{j}$, and $\Tr(\cdot)$ denotes the trace.

It is clear that $\text{KSD}_{A}$ does not characterise probability measures for any fixed $A$.  For example if $P = \mathcal{N}(0,C)$ and $Q = \mathcal{N}(0, C')$ are distinct probability measures such that $A^\intercal C= A^\intercal C'$, then $\mbox{KSD}_{A}(Q,P) = 0$.  This motivates us to consider the worst-case discrepancy taken over all $m$-dimensional subspaces. We define the \emph{Grassmann Kernel Stein Discrepancy} (GKSD) as
\begin{talign}
\label{eq:GKSD}
    \text{GKSD}(Q,P) = \sup_{[A]\in \text{Gr}(d,m)} \text{KSD}_{A}(Q,P).
\end{talign}
The following theorem guarantees that GKSD is able to discriminate distinct distributions. The assumptions on $P,Q$ are standard \citep{gong2021sliced}.

\begin{theorem}\label{thm:KSD_separates}
Let $P,Q$ be Borel measures with continuously differentiable densities $p,q$ both supported on $\bbR^{d}$ with $\lim_{\|x\|_2\rightarrow \infty}q(x) = 0$ and suppose the kernel $k(u,v) = \Psi(\|u-v\|_2)$ is characteristic and bounded with continuous second order partial derivatives. Then the \emph{GKSD} \eqref{eq:GKSD} equals $0$ if and only if $p=q$.
\end{theorem}
\begin{remark}
A radial kernel $k(x,y) = \Psi(\|x-y\|_2)$ is characteristic if and only if it is integrally strictly positive definite \citep{sriperumbudur2011universality}. In particular, the assumptions of Theorem \ref{thm:KSD_separates} hold for Gaussian and Inverse Multiquadric kernels.
\end{remark}

Intuitively, projecting on higher dimensions for $s_p(x)$ will allow us to better account for the correlation between dimensions. With optimal projectors, correlations can be captured regardless of the projection dimension, since both GKSD and maxSKSD can discriminate distinct distributions. In practice, however, optimal projectors are not available, and hence projecting onto 1 dimension can lead to sub-optimal approximation due to weak correlation signals.

\section{GRASSMANN STEIN VARIATIONAL GRADIENT DESCENT}
\label{sec:GSVGD}
Equipped with this new form of kernel discrepancy we can perform a variant of SVGD which we call \emph{Grassmann Stein Variational Gradient Descent} (GSVGD). More specifically, we seek to find 
\begin{talign}
    \label{eq:sliced_objective}
    \phi_{A}^* \in \arg\sup_{\phi \in \mathcal{B}_{k_A}} -\frac{d}{d\epsilon}\mbox{KL}(T_{\sharp}Q, P)\Big|_{\epsilon=0},
\end{talign}
where $T_\sharp = I + \epsilon \phi$, and $I$ is the identity map. The supremum on the RHS of \eqref{eq:sliced_objective} is attained by 
\begin{talign*}
    \phi_{A}^{*}(\cdot) = \bbE_{Q}[A A^\intercal s_{p}(x)k(A^\intercal x,A^\intercal \cdot) + A\nabla_{x_1}k(A^\intercal x,A^\intercal \cdot)],
\end{talign*}
for which it is clear that
\begin{talign*}
-\frac{d}{d\epsilon}\mbox{KL}(T^*_{\sharp}Q, P)\Big|_{\epsilon=0} = \mbox{KSD}_{A}(P,Q), \quad T^*_{\sharp} = I + \epsilon \phi^*_A.
\end{talign*}
To transport the particles efficiently, at every step, we seek to identify the subspace $[A] \in \text{Gr}(d,m)$ with the largest discrepancy between $P$ and $Q$, so that we select
\begin{talign}
    \label{eq:sliced_objective_A}
    [A^*] \in \arg\sup_{[A] \in \text{Gr}(d,m)} \mbox{KSD}_A(P,Q),
\end{talign}
for which $\mbox{KSD}_{A^*}=\mbox{GKSD}(P,Q)$. In practice, the objective (\ref{eq:sliced_objective_A}) can be (approximately) solved by searching for a critical point of $\mbox{KSD}_A(P, Q)$ at which its Riemannian gradient vanishes. This is justified by the following proposition, which states that, when the probability measures $P$ and $Q$ differ in some low-dimensional subspace $[A_0]$, then $[A_0]$ is indeed a solution sought by (\ref{eq:sliced_objective_A}).
\begin{proposition}
    Suppose the conditions in Theorem~\ref{thm:KSD_separates} hold for the kernel $k$ and measures $P, Q$ with associated densities $q, p$. Assume further that the kernel satisfies a smoothness condition Assumption \ref{assumption: smoothness of kernel} in Appendix~\ref{appendix: extra assumption on the kernel}, and that there exists $[A_0] \in \textrm{Gr}(d, m)$ such that $q, p$ admit the following decomposition
    \begin{talign}
        q(x) &\propto q^m(P_0 x) \xi(\Pi_{A_0} x) \label{eq: prop optimal projection, decomposition 1}\\ 
        p(x) &\propto p^m(P_0 x) \xi( \Pi_{A_0} x),
        \label{eq: prop optimal projection, decomposition 2}
    \end{talign}
    where $P_0 \coloneqq A_0 A_0^\intercal$, $\Pi_{A_0} = I_d - A_0 A_0^\intercal$ as before, and $p^m, q^m$ and $\xi$ are smooth, positive functions that are integrable on $\bbR^d$ (i.e.\ they are unnormalised densities). Then $\riemannGrad \mbox{KSD}_A(P,Q) |_{A = A_0} = 0$.
    \label{prop: optimal projection}
\end{proposition}

\begin{remark}
    The assumption on the decomposition (\ref{eq: prop optimal projection, decomposition 1}) and (\ref{eq: prop optimal projection, decomposition 2}) holds when the density $q$ agrees with the target $p$ except in the subspace $[A_0]$. An example is that $q$ is a prior, and $p$ is the posterior induced from $q$ and a likelihood function that depends on $x$ only through $A_0 x$. See Appendix \ref{appendix: lemma and proof of proposition 3} for more discussions and the proof.
\end{remark}

Sequentially solving the optimisation problems \eqref{eq:sliced_objective} and \eqref{eq:sliced_objective_A} every time-step yields a particle transport scheme which terminates if and only if the GKSD between the target distribution and the current particle distribution is zero, which implies $P=Q$ under the conditions of Theorem \ref{thm:KSD_separates}. 

Replacing $Q$ by an empirical distribution of particles $(x_1^0,\ldots, x_N^0)$, and denoting by $(x_1^t, \ldots, x_N^t)$ the set of particles obtained at step $t$, the scheme becomes
\begin{talign}
    \label{eq:discrete_opt_particles}
    x_{i}^{t+1}
    &= x_{i}^{t}
    +\frac{\epsilon}{N}\sum_{j=1}^{N} \big[ A_{t} A_{t}^\intercal s_p(x_j^t) k(A_{t}^\intercal x_{j}^{t},A_{t}^\intercal x_{i}^{t}) \nonumber \\
    &\quad + A_{t} \nabla_{x_1} k(A_{t}^\intercal x_{j}^{t},A_{t}^\intercal x_{i}^{t}) \big]
    ,\quad i=1,\ldots,N,\\
    \label{eq:discrete_opt}
    A_{t} &\in \arg\sup_{[A]\in\text{Gr}(d,m)}\alpha_t([A]),
\end{talign}
where $\alpha_t([A]) = \mbox{KSD}_{A}(Q_t, P)$, for $Q_t(dx) = \frac{1}{N}\sum_{i=1}^N \delta_{x_i^t}(dx)$.

\begin{algorithm}[t!]
    \caption{Grassmann Stein Variational Gradient Descent (GSVGD)}
\label{alg:gsvgd}
\begin{algorithmic}[1]
\State {\bfseries Inputs:} $M$ initialized projectors $[A_{0, l}]\in\mbox{Gr}(d, m)$, $l = 1, \ldots, M$; $N$ initialised particles $\{x^{0}_{i}\}_{i=1}^N$; iteration number $n_\text{epochs}$; step sizes $\epsilon, \delta>0$.
\vspace{1mm}
\For{$t=1,2,\ldots,n_\text{epochs}$}
\State Update particles $x_{i}^{t+1}=x_{i}^{t}+\epsilon \sum_{l=1}^M \widehat{\phi}_{A_{t, l}}(x^t_i)$ for $i=1,\ldots,N$, where
\begin{talign}
    \widehat{\phi}_{A_{t, l}}(x^t_i) 
    &= \frac{1}{N}\sum_{j=1}^{N} \big[ \textcolor{red}{A_{t, l} A_{t, l}^\intercal} s_p(x_j^t) k(\textcolor{red}{A_{t, l}^\intercal} x_{j}^{t},\textcolor{red}{A_{t, l}^\intercal} x_{i}^{t}) \nonumber \\
    &\quad + \textcolor{red}{A_{t, l}} \nabla_{x_1} k(\textcolor{red}{A_{t, l}^\intercal} x_{j}^{t},\textcolor{red}{A_{t, l}^\intercal} x_{i}^{t}) \big].
    \label{eq: gsvgd update} 
\end{talign}
\State Update projectors
\begin{align*}
    A_{t+1, l}=\textrm{R}_{[A_{t, l}]}\left( \delta \Pi_{A_{t, l}} \nabla \alpha_t([A_{t, l}]) + \sqrt{2T \delta} \Pi_{A_{t, l}} \xi_{t, l}\right)
\end{align*}
for $l = 1, \ldots, M$, where $\xi_{t, l} \in \mathbb{R}^{d \times m}$ has i.i.d.\ $\mathcal{N}(0, 1)$ entries, and $\textrm{R}_{[A_{t, l}]}$ is defined in (\ref{eq: polar retr}). 
\EndFor
\State \textbf{Return:} Particles $\{x_{i}^{n_\text{epochs}}\}_{i=1}^N$.
    \end{algorithmic}
\end{algorithm}

\paragraph{Optimal Projections via SDEs:} Proposition~\ref{prop: optimal projection} implies the ``optimal'' projection $A_0$ is \emph{one} critical point of $\alpha_t$, but it does not imply the critical point is unique. Indeed, the problem (\ref{eq:sliced_objective_A}) is non-convex and there
might be multiple local maximisers, from which finding global maxima is generally NP-hard.  A well-known remedy is to add white noise to the gradient flow \citep{aluffi1985global,geman1986diffusions}, allowing the trajectory to escape from local maxima. Mathematically, such methods can be formulated in terms of Stochastic Differential Equations (SDEs).  In this particular instance, we explore the use of SDEs taking values on the Grassmann manifold. We consider stochastic dynamics defined by the solution of the following Stratonovich SDE
\begin{talign}
\label{eq:strat}
dA(t)
&= \mu \riemannGrad \alpha_t([A])\,dt + \sqrt{2T\mu} \Pi_{A_t}\circ dW(t),\\
&= \mu \Pi_{A(t)}\nabla \alpha_t([A(t)])\,dt + \sqrt{2T\mu} \Pi_{A(t)}\circ dW(t)\nonumber
\end{talign}
where $W(t)$ is a $\mathbb{R}^{d\times m}$-valued Brownian motion and $\mu, T > 0$. Combining the particle and projector evolution we can consider the full system as a discretisation of the following coupled ODE-SDE 
\begin{talign*}
    \dot{x}_i(t)  
    &=\frac{1}{N}A(t) \sum_{j=1}^{N} \big[
    \\ & A(t)^\intercal s_p(x_j(t)) k(A(t)^\intercal x_{j}(t),A(t)^\intercal x_{i}(t)) \\
    &+  \nabla_{x_2} k(A(t)^\intercal x_{j}(t),A(t)^\intercal x_{i}(t)) \big],\\
    dA(t)
    &= \mu \Pi_{A(t)}\nabla \alpha_t([A(t)])\,dt + \sqrt{2T\mu} \Pi_{A(t)}\circ dW(t).
\end{talign*}
We see that the parameter $T$ acts as a temperature, facilitating moving between local minima of the function $\alpha$ by annealing the loss function.  The parameter $\mu$ controls the relative timescale between the evolution of $A(t)$ and the particles $(x_1^t\ldots, x_N^t)$.  In particular, when $\mu \gg 1$, the distribution of $A(t)$ is approximated by the quasi-stationary distribution $d\mu_t([A])= e^{\alpha_t([A])/T}d[A]$, where $d[A]$ denotes the canonical measure on the Grassmann manifold.  In the other extreme, when $\mu \ll 1$  the projector $A(t)$ will evolve on a slower timescale than the particles, so that the particles will reach their stationary configuration for a fixed projector before it evolves.

To simulate the Stratonovich SDE \eqref{eq:strat}, we make use of the formulation of diffusions on a Riemannian manifold, making use of the associated retraction map, see \citet{staneva2017learning}, \citet{baxendale1976measures}, and \citet{belopolskaya2012stochastic} where the $\text{Gr}(d,m)$-valued SDE \eqref{eq:strat} can be expressed in the form of a system of Ito SDEs 
\begin{talign}
A(t) &= \textrm{R}_{[A(t)]}\left(B(t)\right),\\
\label{eq:tangent_sde}
dB(t) &=  \mu \Pi_{A(t)}\nabla \alpha_t([A(t)])\,dt + \sqrt{2T\mu}\Pi_{A(t)}dW_t.
\end{talign}
With step-size $\delta > 0$, an Euler-Maruyama discretisation for \eqref{eq:tangent_sde} results in
\begin{talign}
    A_{t+1} = \textrm{R}_{[A_t]}\left(\Pi_{A_t}\nabla \alpha_t([A_t])\,\delta + \sqrt{2T\delta}\Pi_{A_t}\xi_t\right),
    \label{eq: discretized projector update}
\end{talign}
where $\xi_t$ is a $d \times m$ matrix with independent Gaussian entries.  Note that we ``absorb'' the constant $\mu$ into the step-size $\delta$ for simplicity.

\paragraph{Batch Generalisation} 
An important observation is that at each step (\ref{eq:discrete_opt}), $x^{t+1}_i$ is different from $x^{t}_i$ only in the image of $A$. This contrasts with both SVGD as well as S-SVGD, both of which update the particles along all dimensions. One extension is to use more than one projector $A_{t, l} \in \bbR^{m_l \times d}$ for $l = 1, \ldots, M$, where $1 \leq M \leq d$ and $\sum_{l = 1}^M m_l \leq d$. At each optimisation step, we can impose orthonormality using QR factorisation, so that the $d \times d$ matrix $A_t$ formed by column-wise concatenation of $A_{t, 1}, \ldots, A_{t, M}$, is orthogonal. However, this requires $\calO(d^3)$ operations, which can be prohibitive for large $d$. We therefore propose to impose orthonormality every 1000 steps, which we find sufficient in practice. The full algorithm is presented in Algorithm \ref{alg:gsvgd}.

\paragraph{Computational Complexity}
Each iteration of the particles update in Algorithm~\ref{alg:gsvgd} requires $\calO(M N dm(N + d))$ operations. Compared with the $\calO(Nd (N + d))$ cost for S-SVGD, the extra factor $Mm$ arises from the fact that (\ref{eq: gsvgd update}) involves matrix multiplication with $A_{t, l}$, which can be avoided in S-SVGD due to simplification by using the canonical basis. However, this extra computation can be largely reduced by using a fixed number of projectors $M$ and projection dimension $m$. In our experiments, we find $M=20$ projectors are sufficient. A reasonable choice of $m$, however, will depend on the specific problem. Finally, each step of the projector update requires $\calO(M dm^2)$ due to the retraction map, which reduces to the same cost as the slice update in S-SVGD when $M = d$ and $m = 1$.

\section{RELATED WORK}
\label{sec:related_work}
\paragraph{Connections to S-SVGD}
% Both our method and the S-SVGD of \citet{gong2021sliced} are based on KSD with a matrix-valued kernel, and both GKSD and the sliced kernel Stein Discrepancy (SKSD) require optimal projectors or slices in order to discriminate between two distributions. In a followup work, \citet{gong2021active} relaxed the optimality condition of the slices in SKSD, but they did not apply the new KSD variant to particle inference. 

The major differences between GSVGD and S-SVGD are that \emph{(i)} GSVGD allows projecting onto a arbitrary low-dimensional subspace, \emph{(ii)} GSVGD updates the projectors via Riemannian optimisation on the Grassmann manifold, whereas S-SVGD optimises the slices in the unit ball of $\bbR^d$, and \emph{(iii)} the projectors for the score function and the data are set to the same in GSVGD but not in S-SVGD. In particular, S-SVGD slices the score function with a fixed basis $O$ (i.e.\ the canonical basis), and such arbitrary choice may result in inefficient exploration of the state-space if the basis vectors $r \in O$ do not align with the latent dimensions of the target. On the other hand, optimising the projectors for the score function and the data separately would lead to a complex joint optimisation program that is practically challenging. GSVGD strikes a balance by using the same projector for the score function and the data.

\paragraph{Other SVGD Variants}
\citet{liu2018riemannian} extends SVGD to tasks defined on Riemannian manifolds and \cite{shi2021sampling} proposes a variant applied to non-Euclidean spaces (such as a simplex) via mirror descent. These are fundamentally different from our work, which concerns problems on the Euclidean space, while evolving the projectors on the Grassmann manifold. Closer to our method is the pSVGD of \citet{chen2020projected}, where the particles are projected onto a subspace defined by the top eigenvectors of a gradient information matrix. This method however suffers from practical issues, e.g.\ finding the eigen-decomposition would incur a computational cost that is cubic in the dimension. Other attempts to alleviate the curse-of-dimensionality problem of SVGD include Message Passing SVGD \citep{zhuo2018message} and graphical SVGD \citep{wang2018stein}, both of which are limited to problems where the target distribution is defined on a \emph{probabilistic graphical model} (PGM) with a known Markov structure. In contrast, GSVGD can be applied to distributions of an arbitrary form. 

% \citet{liu2018riemannian} extends SVGD to tasks defined on Riemannian manifolds and \cite{shi2021sampling} proposes a variant applied to non-Euclidean spaces (such as a simplex) via mirror descent. These are fundamentally different from our work, which concerns problems on the Euclidean space, while evolving the projectors on the Grassmann manifold. Closer to our method is the pSVGD of \citet{chen2020projected}, where the particles are projected onto a subspace defined by the top eigenvectors of a gradient information matrix. This method however suffers from some practical issues, e.g.\ finding the eigen-decomposition would incur a computational cost that is cubic in the dimension, and implementing this algorithm is non-trivial. For these reasons, we did not include pSVGD in our experiments. Other attempts to address the curse-of-dimensionality problem of SVGD include Message Passing SVGD \citep{zhuo2018message} and graphical SVGD \citep{wang2019stein}, both of which assume the target is defined on a \emph{probabilistic graphical model} (PGM) with a \emph{known} Markov structure, and hence are limited to such situations. In contrast, GSVGD can be applied to distributions of an arbitrary form. 
% 

\section{EXPERIMENTS}
\label{sec: experiments}
We study the uncertainty quantification property of GSVGD against existing methods such as SVGD and S-SVGD on a number of experiments. We conclude from our experiments that SVGD can underestimate the uncertainty, S-SVGD can overestimate it and GSVGD, with appropriate projection dimensions, yields the best estimates. For GSVGD, we use the batch approach of Algorithm~\ref{alg:gsvgd}, where at most $M=20$ projectors with the same projection dimension $m$ are used. The temperature $T$ is set to a small value and is gradually incremented to a larger one via an annealing scheme. The learning rates for all methods are tuned for optimal performance on the multimodal mixture example, and the same choices are used for the other experiments.

In all experiments, we use the Gaussian RBF kernel parameterised by $k(x, x^\prime) = \exp\left(  -\| x - x^\prime \|_2^2/(2\sigma^2)\right)$ with bandwidth $\sigma^2 = \texttt{med}^2 / (2 \log n)$, where $\texttt{med}$ is the median of the pairwise distance of the particles. This follows the heuristic in \citet{Liu2016SVGD}. To quantify sample quality, we use the \emph{energy distance} \citep{szekely2013energy} between the approximation and the ground truth. However, we find that the energy distance is not sensitive to differences in correlations between multivariate samples. Therefore, we also evaluate how well each method captures the dependence between multivariate samples by the covariance estimation error $||\hat{\Sigma} - \Sigma||_{2}$, where $||\cdot||_2$ is the Frobenius norm, $\Sigma$ is the ground truth sample covariance matrix given by a Hamiltonian Monte Carlo (HMC) (for non-synthetic data), and $\hat{\Sigma}$ is the estimated covariance. For further details, see Appendix~\ref{appendix: experimental details}. All the code is available at \url{https://github.com/ImperialCollegeLondon/GSVGD}.

\begin{figure}
    % \vspace{-3mm}
    \centering
    \includegraphics[width=0.23\textwidth]{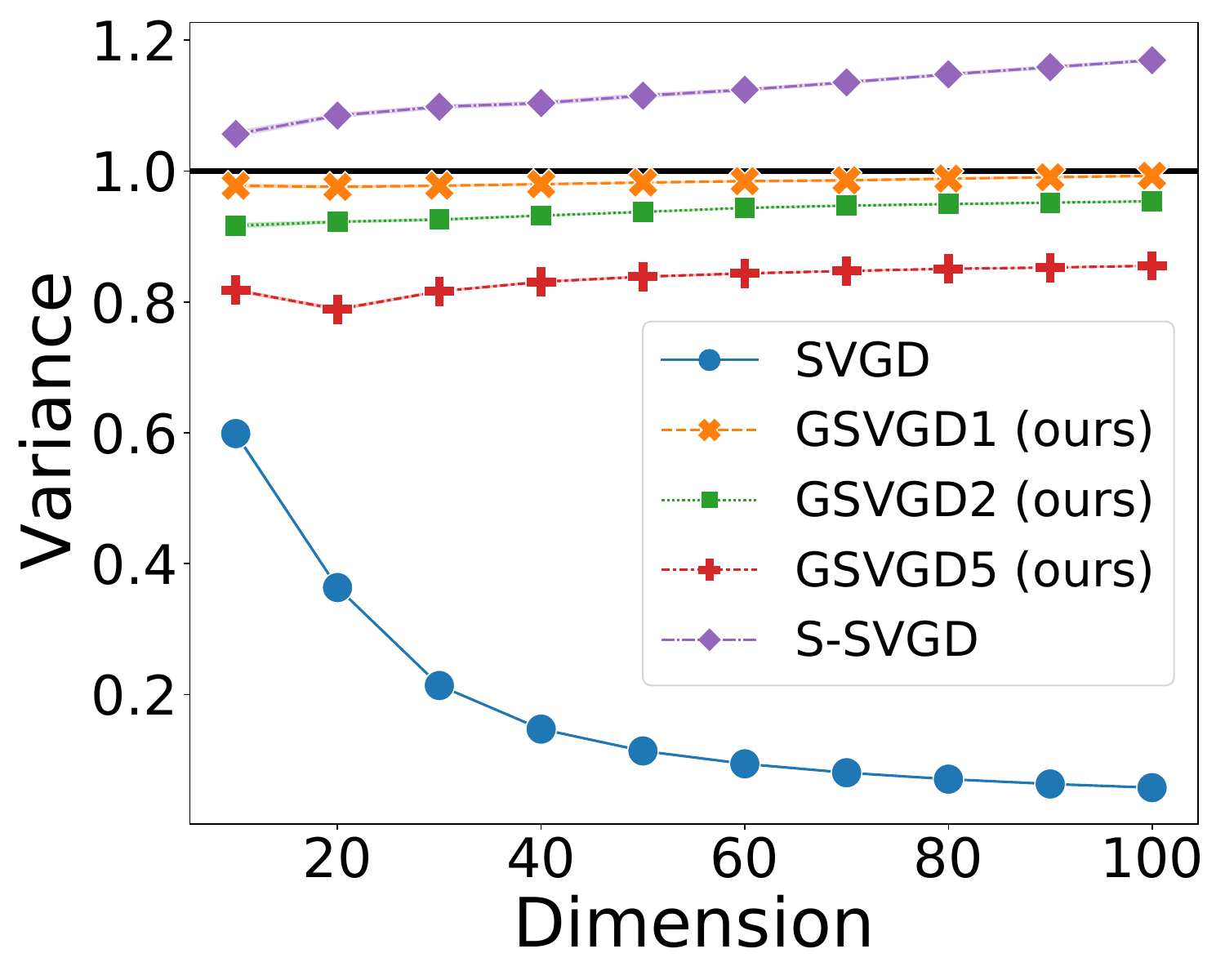}
    \vspace{-3mm}
    \caption{Estimating the dimension-averaged marginal variance of $p(x) = \calN(x; 0, I_d)$ across different dimensions $d$. GSVGD1 means GSVGD with 1-dimensional projections. Black solid line marks the true value. Results are averaged over 20 repetitions.}
    \label{fig: gaussian marginal var}
    % \vspace{-5mm}
\end{figure}

\paragraph{Multivariate Gaussian}
The first toy example is a $d$-dimensional multivariate Gaussian $p(x) = \calN(x; 0, I_d)$. For each method, 500 particles are initialized from $\calN(x; 2 \mathds{1}, 2I_d)$, where $\mathds{1} \in \bbR^d$ denotes the vector of ones. This is a standard benchmark used to illustrate the diminishing variance issue of SVGD \citep{Liu2016SVGD, gong2021sliced, zhuo2018message}. We show empirically that GSVGD is able to mitigate this problem by plotting the dimension-averaged estimates for the marginal variance in Figure~\ref{fig: gaussian marginal var}. We see that GSVGD estimates stay roughly at 1 for all dimensions as expected, whereas the SVGD estimates scale inversely with dimension. S-SVGD, on the other hand, over-estimates the variance. We remark, however, that GSVGD5 appears to be biased with a lower variance. This is due to the variance underestimation issue of SVGD persisting in 5 dimensional subspaces.

\paragraph{Multimodal Mixture}
In the second example, the target distribution is a mixture of $4$ $d$-dimensional Gaussian distributions $p(x) = \sum_{k = 1}^{4} 0.25 \mathcal{N}(x; \mu_k, I_d)$ with uniform mixture ratios. The first two coordinates of the mean vectors are equally spaced on a circle, while the other coordinates are set to 0; see Figure~\ref{fig: multimodal particles 50d} and Appendix~\ref{appendix:multimodal}. Particles are initialized from $\calN(0, I_d)$ and only the first two dimensions need to be learned. The primary goal is to investigate whether GSVGD can efficiently recover this low-dimensional latent structure by adapting the projectors. As shown in \ref{fig: multimodal energy}, GSVGD1 and GSVGD2 outperforms S-SVGD in all dimensions in terms of the energy distance, while GSVGD5 is able to achieve a competitive performance. SVGD, on the other hand, fails to capture the uncertainty of the target, as shown in Figure~\ref{fig: multimodal particles 50d}. 

% That is, the first two coordinates are equally spaced on a circle of radius $\sqrt{5}$ in $\bbR^2$, while the others are centred at the origin 
\begin{figure}[t!]
    \centering
    % \vspace{-6mm}
    \subfigure[Multimodal mixture]
    {
        \label{fig: multimodal particles 50d}
        \includegraphics[width=0.22\textwidth]{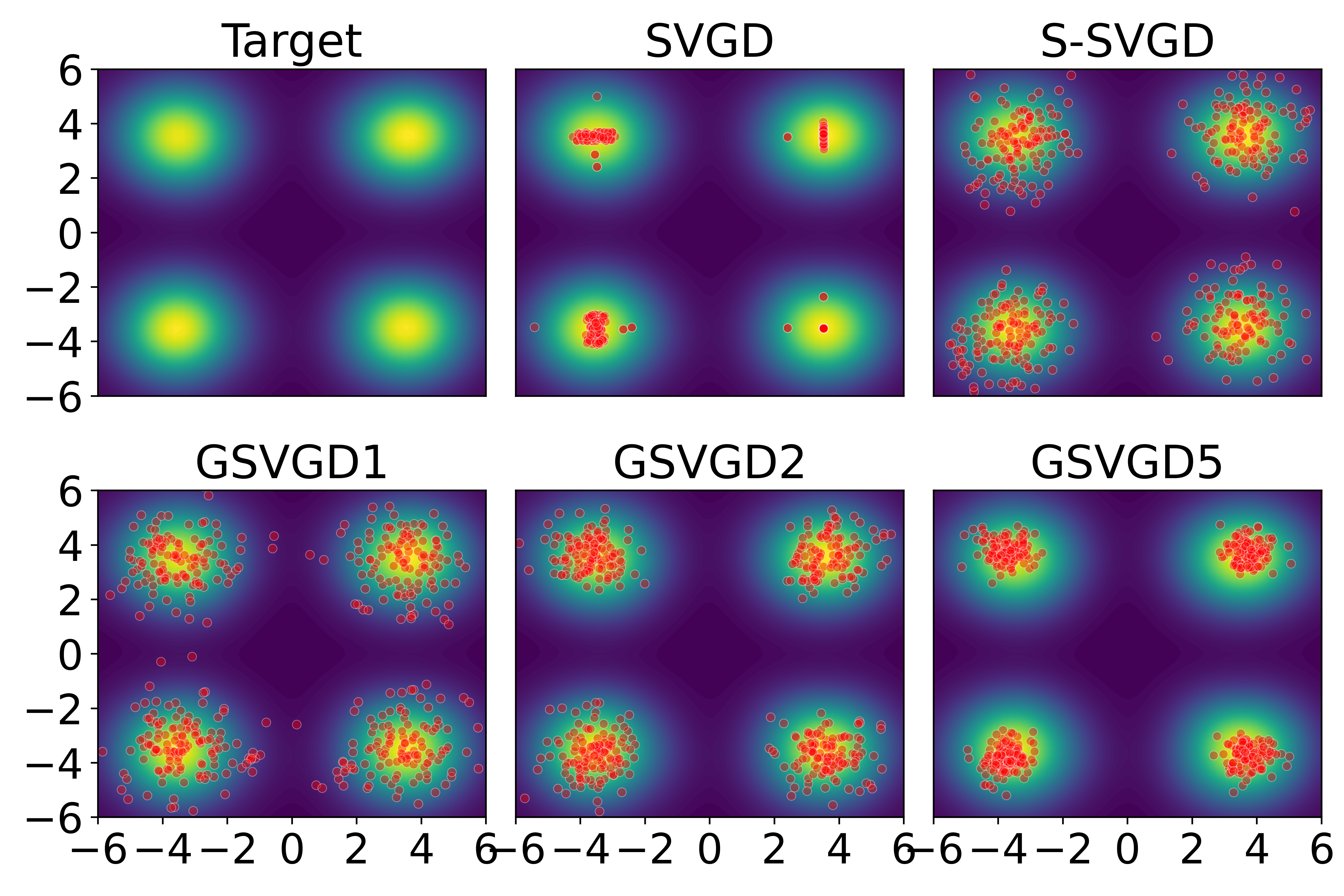}
    }
    \subfigure[X-shaped mixture]
    {
        \label{fig: xshaped particles 50d}
        \includegraphics[width=0.22\textwidth]{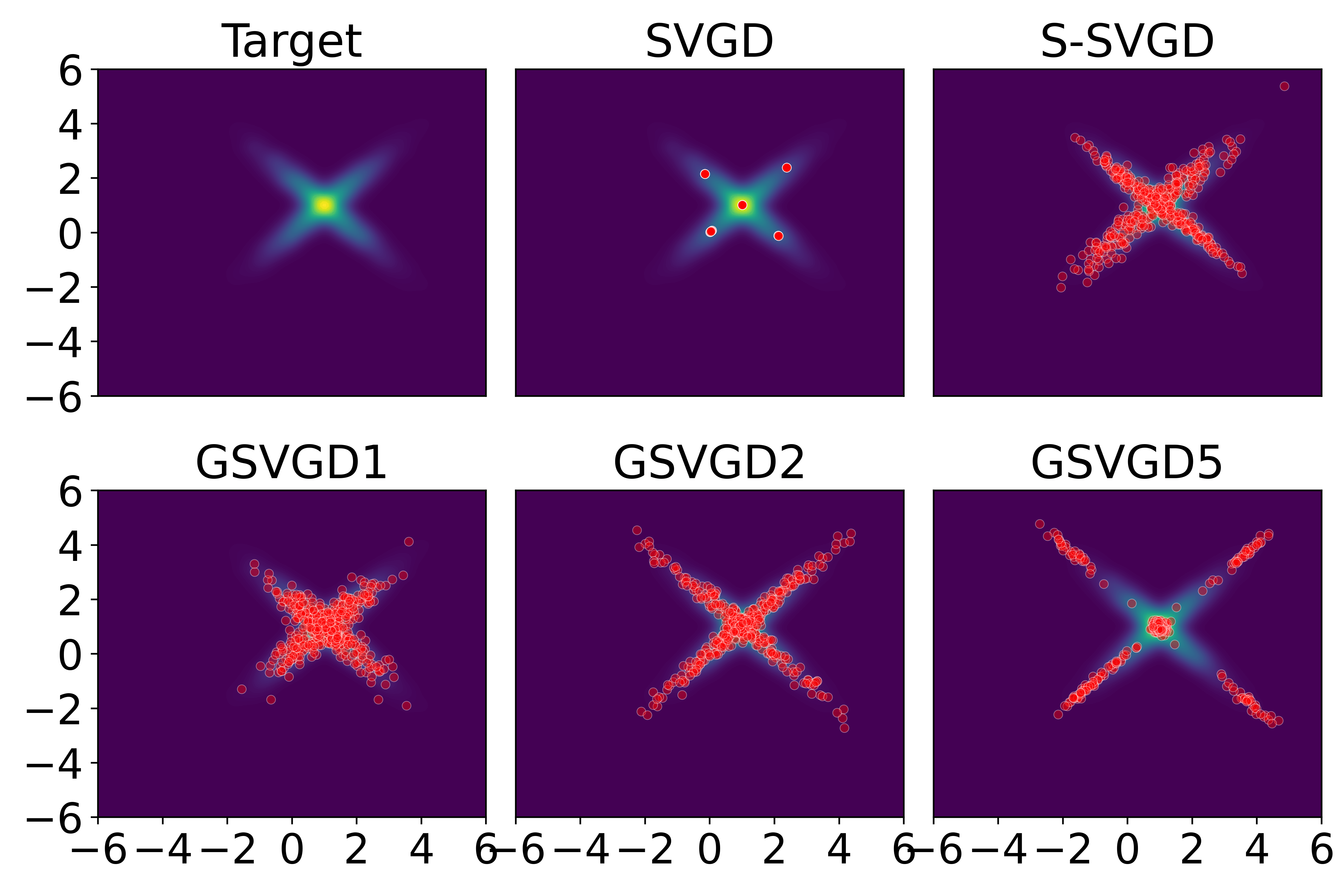}
    }
    \vspace{-4mm}
    \caption{Marginals of target and estimated densities in the first 2 dimensions. Red dots are the particles after 2000 iterations and the target density is shown by the contour. Both targets have dimension 50.}
    \label{fig: final particles}
\end{figure}

\begin{figure}[t!]
    \centering
    % \vspace{-6mm}
    \subfigure[Multimodal mixture]
    {
        \label{fig: multimodal energy}
        \includegraphics[width=0.22\textwidth]{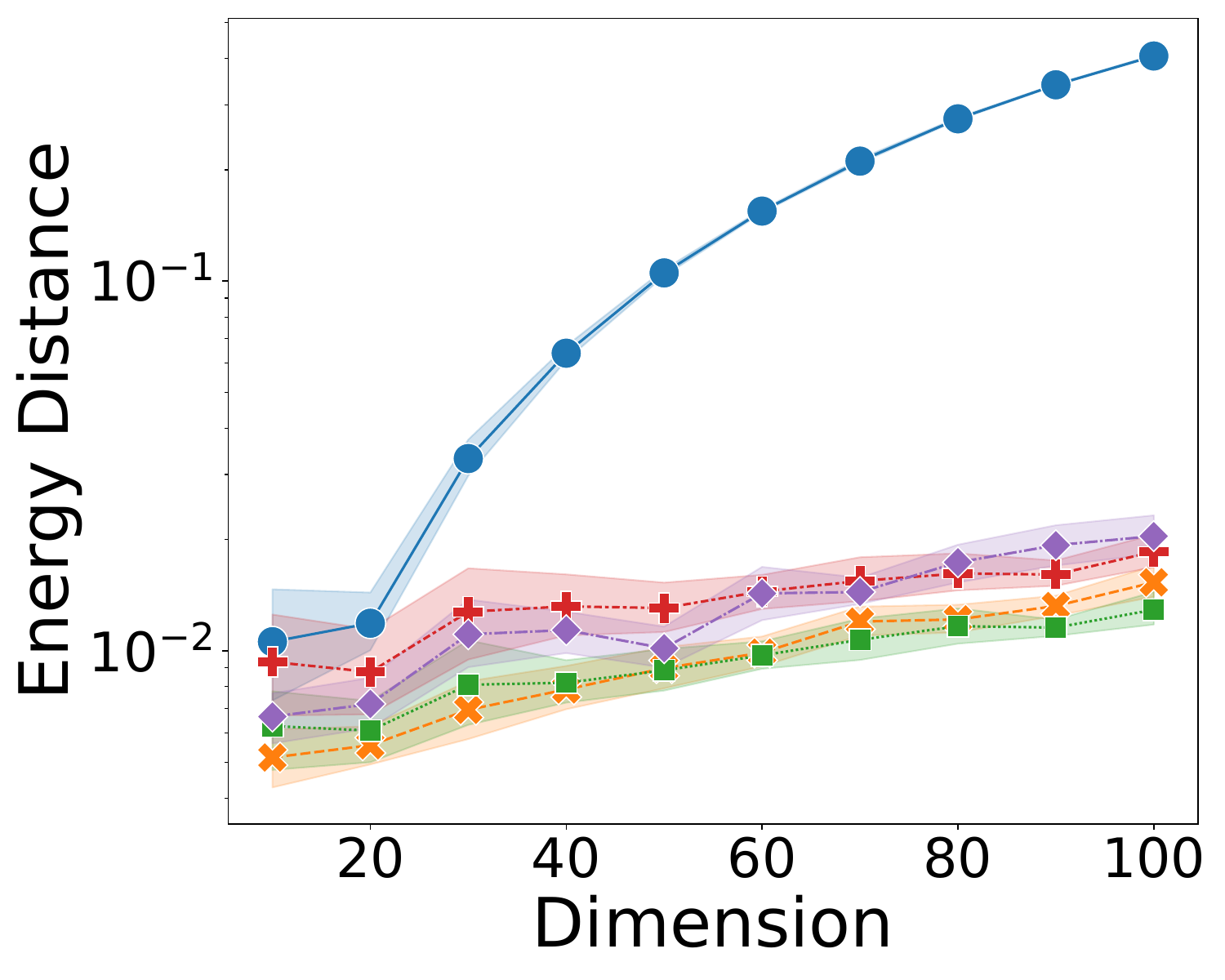}
    }
    \subfigure[X-shaped mixture]
    {
        \label{fig: xshaped energy}
        \includegraphics[width=0.22\textwidth]{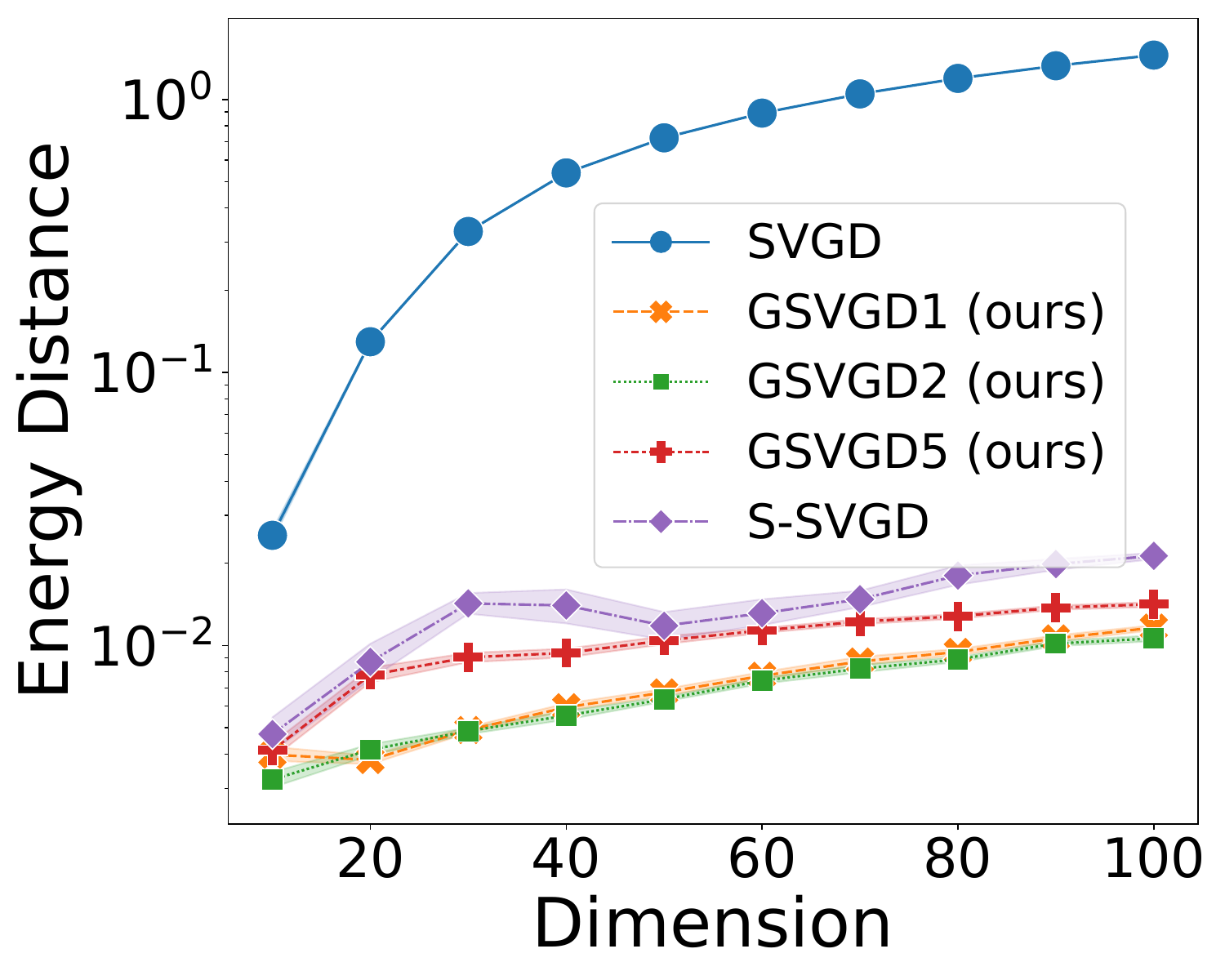}    
    }
    % \subfigure[X-shaped mixture convergence \xl{move to appendix?}]
    % {
    %     \label{fig: xshaped energy vs epochs}
    %     \includegraphics[width=0.26\textwidth]{figs/xshaped/energy_50d.pdf}
    % }
    \vspace{-4mm}

    \caption{Energy distance between the target distribution and the particle estimation. Results are averaged over 20 repetitions. $95\%$-confidence intervals are shown by the shaded regions.}
    % (c): Energy distance against iterations in the 50-dimensional X-shaped example.}
    \label{fig: summary metrics}
    % \vspace{-5mm}
\end{figure}

\paragraph{X-Shaped Mixture}
The target in this experiment is a $d$-dimensional mixture of two correlated Gaussian distributions $p(x) = 0.5 \mathcal{N}(x; \mu_1, \Sigma_1) + 0.5 \mathcal{N}(x; \mu_2, \Sigma_2)$. The means $\mu_1, \mu_2$ of each Gaussian have components equal to 1 in the first two coordinates and 0 otherwise, and the covariance matrices admit a correlated block diagonal structure (see Appendix~\ref{appendix:xshaped}). The mixture hence manifests as an ``X-shaped'' density marginally in the first two dimensions (see Figure~\ref{fig: xshaped particles 50d}). Figure~\ref{fig: xshaped energy} shows that GSVGD1 and GSVGD2 achieve a better approximation quality than both SVGD and S-SVGD across all dimensions. This is further verified by plotting the first two coordinates of the final particles in \ref{fig: final particles} (b). On the other hand, GSVGD5 shows a slightly suboptimal performance for dimensions larger than 30. This is potentially because of the inefficiency due to projecting the particles to 5 dimensional subspace in the case we only care about the first 2.

\paragraph{Conditioned Diffusion Process}

\begin{figure}[t!]
    \centering
    % \vspace{-5mm}
    \includegraphics[width=0.48\textwidth]{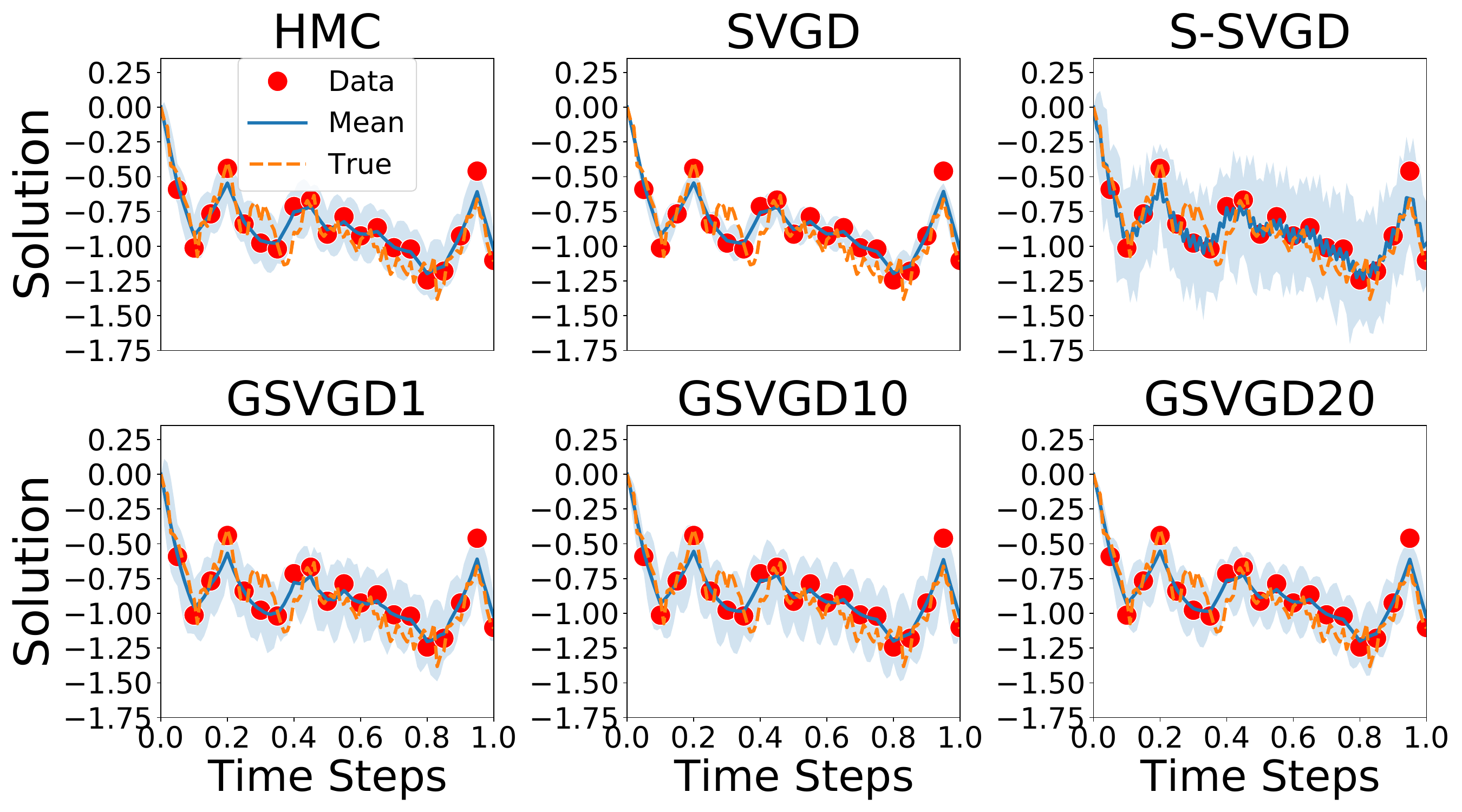}
    \vspace{-8mm}
    \caption{HMC, SVGD, S-SVGD and GSVGD solutions, as well as the posterior mean and $95\%$ confidence interval after 1000 iterations. }
    \label{fig: diffusion solution}
    % \vspace{-4mm}
\end{figure}

The next example is a benchmark that is often used to test inference methods in high dimensions \citep{cui2016dimension, chen2020projected, detommaso2018stein}. We consider a stochastic process $u: [0, 1] \to \bbR$ governed by
\begin{talign}
    du_t = \frac{10u(1 - u^2)}{1 + u^2} dt + dx_t, 
    \quad
    u_t = 0,
    \label{eq: diffusion process}
\end{talign}
where $t \in (0, 1]$, and the forcing term $x = (x_t)_{t \geq 0}$ follows a Brownian motion so that $x \sim q = \calN(0, C)$ with $C(t, t') = \min(t, t')$. We observe noisy data $y = (y_{t_1}, \ldots, y_{t_{20}})^\intercal \in \bbR^{20}$ at 20 equi-spaced time points $t_i = 0.05i$, where $y_{t_i} = u_{t_i} + \epsilon$ for $\epsilon \sim \calN(0, \sigma^ 2)$ with $\sigma = 0.1$, and $u_{t_i}$ is generated by solving (\ref{eq: diffusion process}) at a true Brownian path $x_{\textrm{true}}$. The goal is to use $y$ to infer the forcing term $x$ and thereby the solution state $u$. 

The prior we use for $x$ is the Brownian motion described above. The dynamic is discretized using an Euler-Maruyama scheme with step size $\Delta t = 10^ {-2}$, leading to a 100-dimensional inference problem. Due to the smoothness of the Brownian path, we expect the solution $u$ to have an intrinsic low-dimensional structure. Figure~\ref{fig: diffusion solution} shows the mean and 95$\%$ credible intervals of the particle estimation of $u$ after 1000 iterations. Compared with a reference HMC, all methods are able to estimate the mean accurately, but SVGD and S-SVGD either underestimate or overestimate the credible intervals. GSVGD projecting on $20$ dimensions gives the most accurate uncertainty estimates. Choosing a good projection dimension is non-trivial even in this simple problem; however, by plotting the energy distance between the HMC reference and the results against different projection dimensions (Figure \ref{fig: diffusion energy} of Appendix), we see that so long as we do not project onto very low or high dimensions, GSVGD achieves a greater degree of agreement with the reference HMC than SVGD and S-SVGD. 

\paragraph{Bayesian Logistic Regression} 
Following \cite{Liu2016SVGD}, we consider the Bayesian logistic regression model applied to the Covertype dataset \citep{asuncion2007uci}. We use HMC posterior samples as a gold standard for evaluating the posterior approximations. We ran all methods on a subset consisting 1,000 randomly selected data points 10 times. Similar to the synthetic experiments, we see in Figure~\ref{fig:blr:covariance_estimation} that there exists an optimal projection dimension (around $m=10$) such that GSVGD more accurate approximates the covariance matrix whilst yielding similar energy distances as S-SVGD. S-SVGD severely overestimates the uncertainty, also shown in Figure~\ref{fig:blr:cov.pdf}, whereas SVGD severely underestimates it.

\begin{figure}[t!]
\centering
% \vspace{-5mm}
\subfigure[Energy distance against projection dimensions]{
\includegraphics[width=0.22\textwidth]{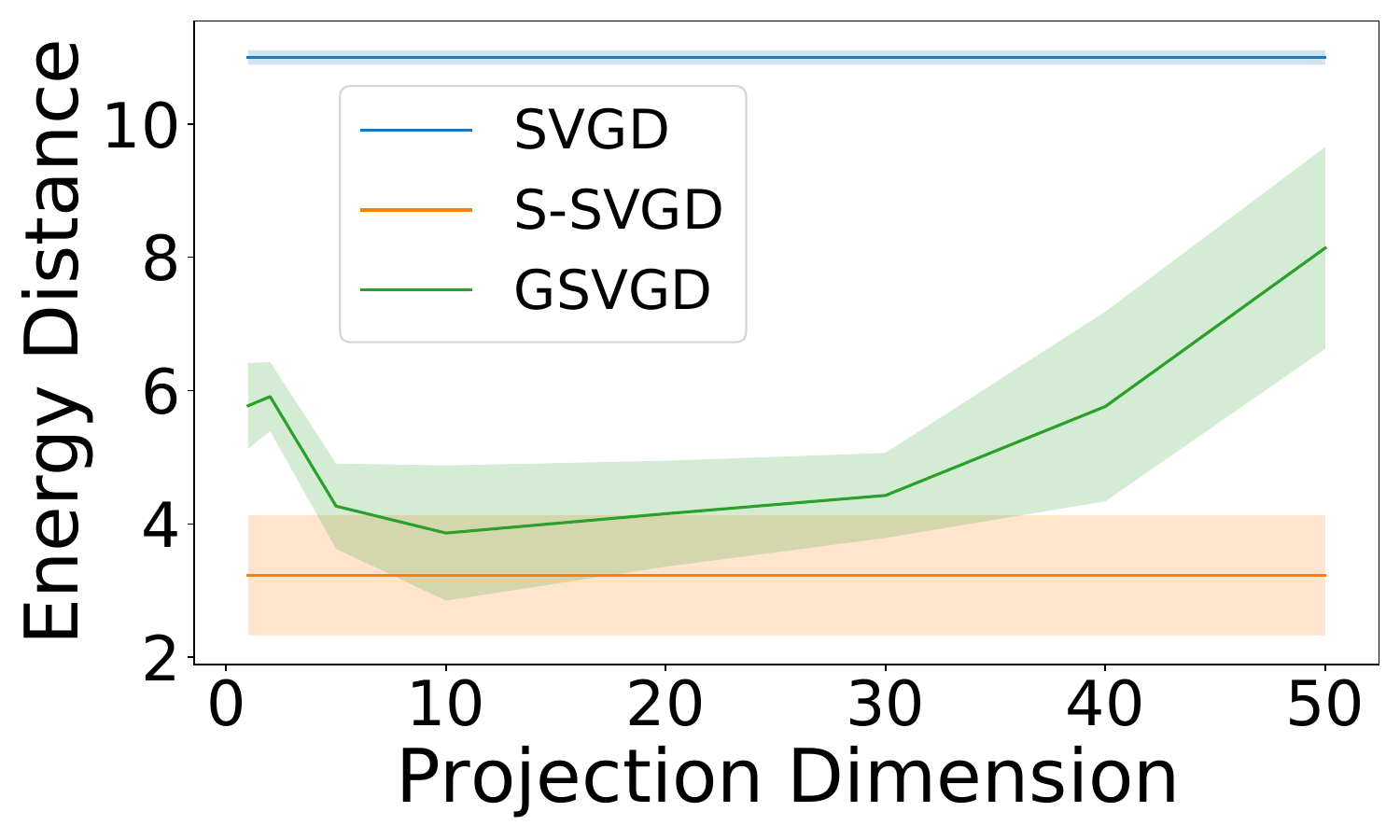}
}
\label{fig:blr:energy distance}
\subfigure[$||\hat{\Sigma}-\Sigma||_2$ against projection dimensions]{
\includegraphics[width=0.22\textwidth]{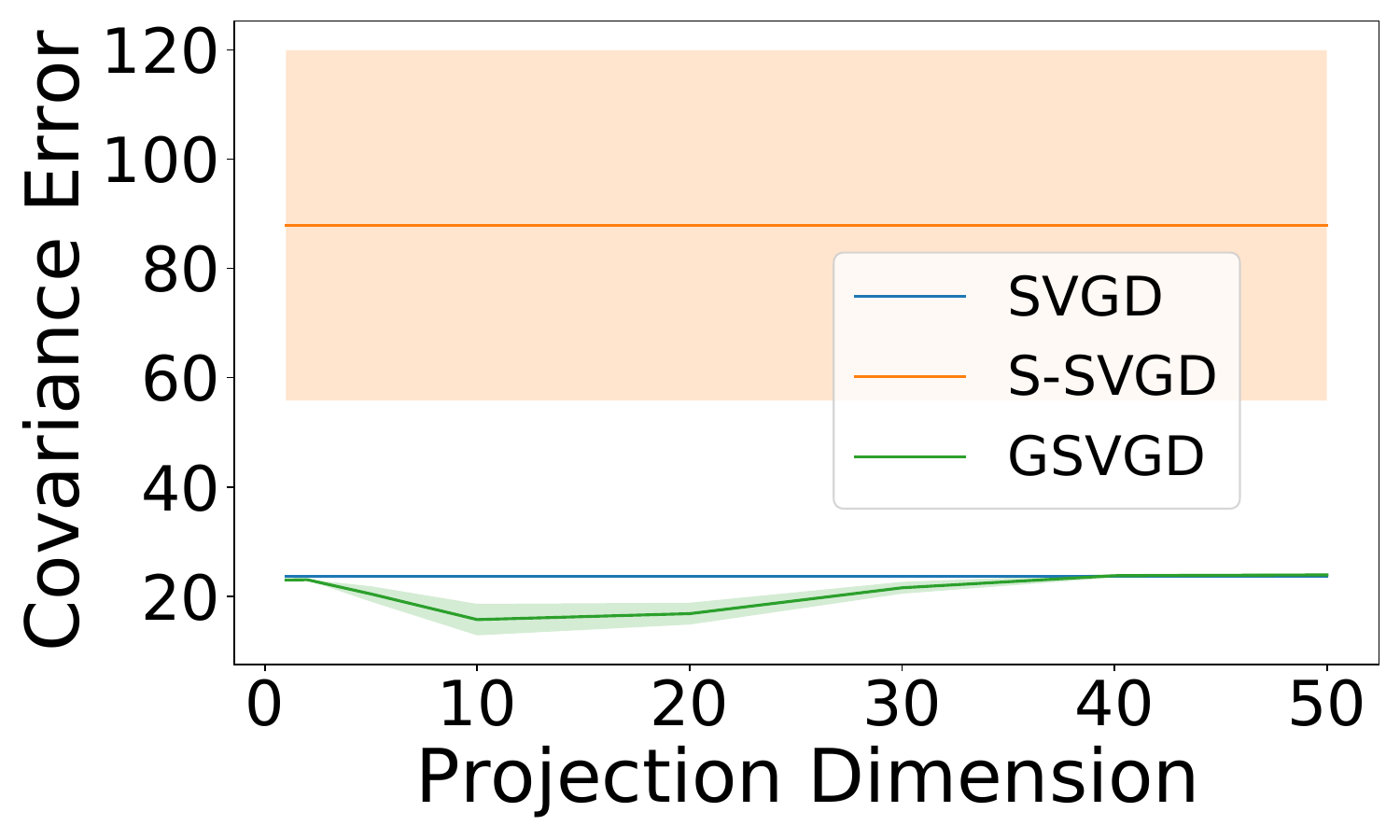}
}
\vspace{-3mm}
\caption{Metrics for Covertype.}
\label{fig:blr:covariance_estimation}
\end{figure}

\section{CONCLUSION}
\label{sec: conclusion}
We introduced Grassmann Stein Variational Gradient Descent (GSVGD), a novel extension to the SVGD algorithm, to tackle the curse-of-dimensioanlity problem. At each step, GSVGD seeks the subspace in which the proposal and the target has the largest discrepancy measured by the Grassmann Kernel Stein Discrepancy (GKSD). The evolution of the particles and the projector is determined through a coupled ODE-SDE system, which allows trade-offs between exploration and exploitation. The GSVGD algorithm can be easily extended to a batched version that uses more than one projector. Our experiments demonstrate that GSVGD is able to achieve improved performance over existing methods when the target distribution has an intrinsic low-dimensional structure, especially more accurately quantifying the epistemic uncertainty.

\paragraph{Limitations of GSVGD:} One limitation of GSVGD is the choice of projection dimension and number of projectors to use in order to attain superior performance. One possible way is to perform a grid search over a range of feasible values and then assess the performance using metrics that measure particle diversity. We leave this as an open question for future work. In addition, our Grassmann-valued SDE introduces additional hyperparameters, which potentially require tuning on a case-by-case basis for each problem. In Section~\ref{sec: experiments} we present some heuristics for tuning GSVGD hyperparameters that we found yielded consistent results, but we would recommend future work to study the sensitivity of the hyperparameters.
% Possible future works include a theoretical study on the convergence of the coupled particle-projector evolution...

% Possible future works include theoretical studies on the dynamic of the particles-projector evolution, as well as robustness analyses with respect to different choices of hyper-parameters (learning rates for the particles and projectors; temperature $T$). Furthermore, applying the discrepancy GKSD to tasks such as goodness-of-fitting tests \citep{Liu2016, gretton2012kernel, jitkrittum2017linear} or model learning \citep{grathwohl2020cutting, feng2017learning, pu2017vae} is also an interesting direction.

\begin{figure}[t!]
\centering
% \vspace{-4mm}
\includegraphics[width=0.45\textwidth]{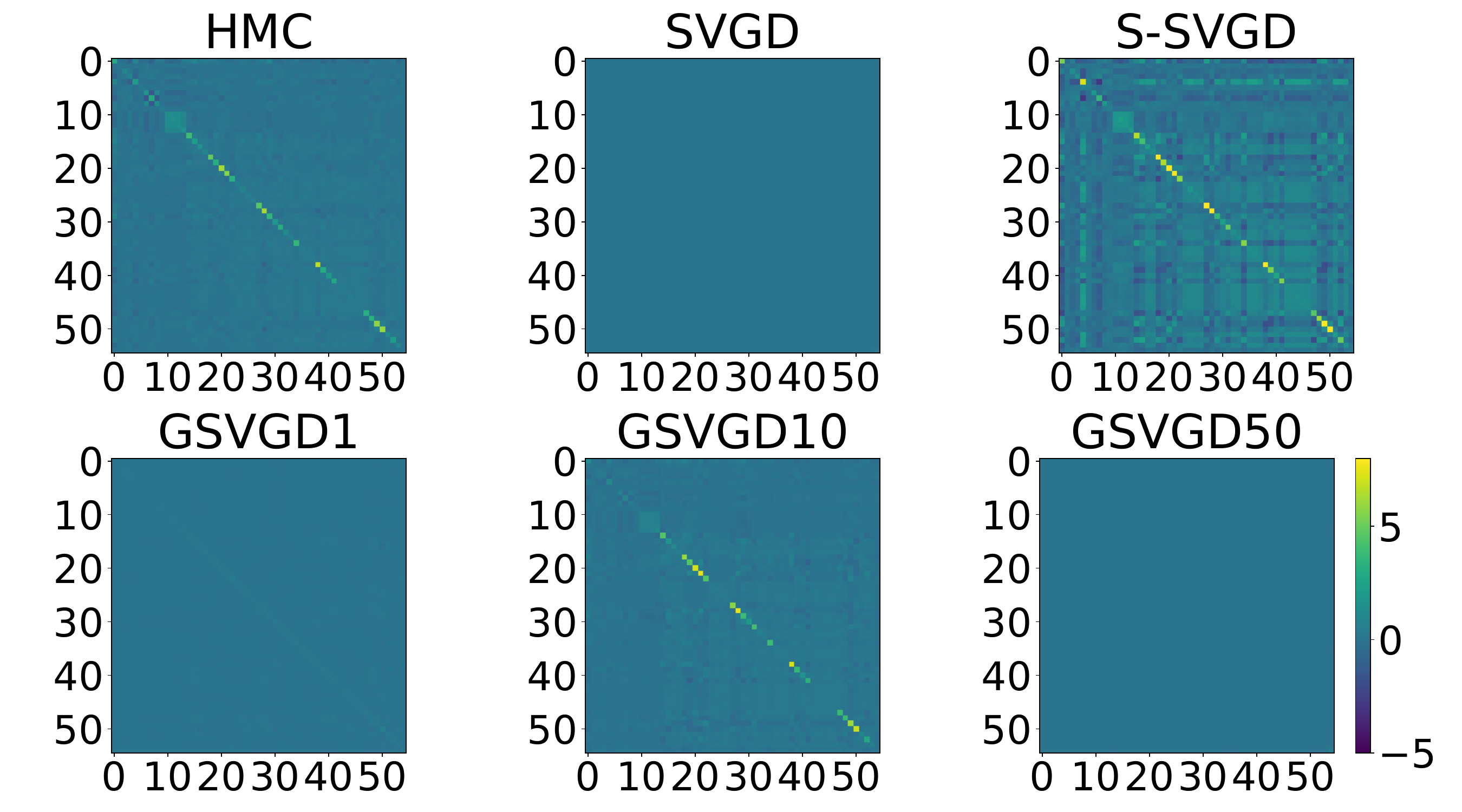}
\vspace{-3mm}
\caption{Covariance matrices for Covertype.}
% \vspace{-5mm}
\label{fig:blr:cov.pdf}
\end{figure}

\subsubsection*{Acknowledgements}
XL was supported by the President’s PhD Scholarships of Imperial College London and the EPSRC StatML CDT programme EP/S023151/1. HZ was supported by the EPSRC Centre for Doctoral Training in Modern Statistics and Statistical Machine Learning EP/S023151/1, the Department of Mathematics of Imperial College London and Cervest Limited. JFT was supported by the EPSRC and MRC through the OxWaSP CDT programme EP/L016710/1. GW was supported by an EPSRC Industrial CASE award EP/S513635/1 in partnership with Shell UK Ltd. AD work was supported by Wave 1 of The UKRI Strategic Priorities Fund under the EPSRC Grant EP/W006022/1, particularly the ``Foundations of Ecosystems of Digital Twins'' theme within that grant \& The Alan Turing Institute

% references ----------------------------------------------------------------------------------

\bibliography{ref}

\begin{thebibliography}{}

\bibitem[Absil et~al., 2008]{AbsMahSep2008}
Absil, P.-A., Mahony, R., and Sepulchre, R. (2008).
\newblock {\em {Optimization Algorithms on Matrix Manifolds}}.
\newblock Princeton University Press, Princeton, NJ.

\bibitem[Aluffi-Pentini et~al., 1985]{aluffi1985global}
Aluffi-Pentini, F., Parisi, V., and Zirilli, F. (1985).
\newblock {Global optimization and stochastic differential equations}.
\newblock {\em Journal of optimization theory and applications}, 47(1):1--16.

\bibitem[Asuncion and Newman, 2007]{asuncion2007uci}
Asuncion, A. and Newman, D. (2007).
\newblock {UCI machine learning repository}.

\bibitem[Ba et~al., 2022]{ba2022understanding}
Ba, J., Erdogdu, M.~A., Ghassemi, M., Sun, S., Suzuki, T., Wu, D., and Zhang,
  T. (2022).
\newblock {Understanding the Variance Collapse of {SVGD} in High Dimensions}.
\newblock In {\em International Conference on Learning Representations}.

\bibitem[Baxendale et~al., 1976]{baxendale1976measures}
Baxendale, P. et~al. (1976).
\newblock {Measures and Markov processes on function spaces}.
\newblock {\em M{\'e}moires de la Soci{\'e}t{\'e} Math{\'e}matique de France},
  46(131-141):3.

\bibitem[Belopolskaya and Dalecky, 2012]{belopolskaya2012stochastic}
Belopolskaya, Y.~I. and Dalecky, Y.~L. (2012).
\newblock {\em {Stochastic equations and differential geometry}}, volume~30.
\newblock Springer Science \& Business Media.

\bibitem[Bendokat et~al., 2020]{bendokat2020grassmann}
Bendokat, T., Zimmermann, R., and Absil, P.-A. (2020).
\newblock {A Grassmann Manifold Handbook: Basic Geometry and Computational
  Aspects}.
\newblock {\em arXiv preprint arXiv:2011.13699}.

\bibitem[Blei et~al., 2017]{blei2017variational}
Blei, D.~M., Kucukelbir, A., and McAuliffe, J.~D. (2017).
\newblock {Variational inference: A review for statisticians}.
\newblock {\em Journal of the American statistical Association},
  112(518):859--877.

\bibitem[Boumal, 2020]{boumal2020introduction}
Boumal, N. (2020).
\newblock An introduction to optimization on smooth manifolds.
\newblock {\em Available online, May}.

\bibitem[Chen and Ghattas, 2020]{chen2020projected}
Chen, P. and Ghattas, O. (2020).
\newblock {Projected Stein Variational Gradient Descent}.
\newblock In Larochelle, H., Ranzato, M., Hadsell, R., Balcan, M.~F., and Lin,
  H., editors, {\em Advances in Neural Information Processing Systems},
  volume~33, pages 1947--1958. Curran Associates, Inc.

\bibitem[Chwialkowski et~al., 2016]{Chwialkowski16}
Chwialkowski, K., Strathmann, H., and Gretton, A. (2016).
\newblock A kernel test of goodness of fit.
\newblock In Balcan, M.~F. and Weinberger, K.~Q., editors, {\em Proceedings of
  The 33rd International Conference on Machine Learning}, volume~48 of {\em
  Proceedings of Machine Learning Research}, pages 2606--2615, New York, New
  York, USA. PMLR.

\bibitem[Cui et~al., 2016]{cui2016dimension}
Cui, T., Law, K.~J., and Marzouk, Y.~M. (2016).
\newblock {Dimension-independent likelihood-informed MCMC}.
\newblock {\em Journal of Computational Physics}, 304:109--137.

\bibitem[Dang et~al., 2019]{dang2019hamiltonian}
Dang, K.-D., Quiroz, M., Kohn, R., Minh-Ngoc, T., and Villani, M. (2019).
\newblock {Hamiltonian Monte Carlo with energy conserving subsampling}.
\newblock {\em Journal of machine learning research}, 20.

\bibitem[Detommaso et~al., 2018]{detommaso2018stein}
Detommaso, G., Cui, T., Marzouk, Y., Spantini, A., and Scheichl, R. (2018).
\newblock {A Stein variational Newton method}.
\newblock In Bengio, S., Wallach, H., Larochelle, H., Grauman, K.,
  Cesa-Bianchi, N., and Garnett, R., editors, {\em {Advances in Neural
  Information Processing Systems}}, volume~31. Curran Associates, Inc.

\bibitem[Geman and Hwang, 1986]{geman1986diffusions}
Geman, S. and Hwang, C.-R. (1986).
\newblock Diffusions for global optimization.
\newblock {\em SIAM Journal on Control and Optimization}, 24(5):1031--1043.

\bibitem[Gilks et~al., 1995]{gilks1995markov}
Gilks, W., Richardson, S., and Spiegelhalter, D. (1995).
\newblock {\em Markov Chain Monte Carlo in Practice}.
\newblock Chapman \& Hall/CRC Interdisciplinary Statistics. Taylor \& Francis.

\bibitem[Gong et~al., 2021]{gong2021sliced}
Gong, W., Li, Y., and Hern{\'a}ndez-Lobato, J.~M. (2021).
\newblock {Sliced Kernelized Stein Discrepancy}.
\newblock In {\em International Conference on Learning Representations}.

\bibitem[Hoffman et~al., 2014]{hoffman2014no}
Hoffman, M.~D., Gelman, A., et~al. (2014).
\newblock {The No-U-Turn sampler: adaptively setting path lengths in
  Hamiltonian Monte Carlo.}
\newblock {\em J. Mach. Learn. Res.}, 15(1):1593--1623.

\bibitem[Kolouri et~al., 2019]{kolouri2019generalized}
Kolouri, S., Nadjahi, K., Simsekli, U., Badeau, R., and Rohde, G. (2019).
\newblock {Generalized Sliced Wasserstein Distances}.
\newblock {\em Advances in Neural Information Processing Systems}, 32:261--272.

\bibitem[Levin and Peres, 2017]{levin2017markov}
Levin, D.~A. and Peres, Y. (2017).
\newblock {\em {Markov chains and mixing times}}, volume 107.
\newblock American Mathematical Soc.

\bibitem[Liu and Zhu, 2018]{liu2018riemannian}
Liu, C. and Zhu, J. (2018).
\newblock {R}iemannian {S}tein variational gradient descent for {B}ayesian
  inference.
\newblock In {\em The 32nd AAAI Conference on Artificial Intelligence}, pages
  3627--3634, New Orleans, Louisiana USA. AAAI press.

\bibitem[Liu et~al., 2016]{Liu2016}
Liu, Q., Lee, J., and Jordan, M. (2016).
\newblock {A Kernelized Stein Discrepancy for Goodness-of-fit Tests}.
\newblock In Balcan, M.~F. and Weinberger, K.~Q., editors, {\em Proceedings of
  The 33rd International Conference on Machine Learning}, volume~48 of {\em
  Proceedings of Machine Learning Research}, pages 276--284, New York, New
  York, USA. PMLR.

\bibitem[Liu and Wang, 2016]{Liu2016SVGD}
Liu, Q. and Wang, D. (2016).
\newblock {Stein Variational Gradient Descent: A General Purpose Bayesian
  Inference Algorithm}.
\newblock In Lee, D., Sugiyama, M., Luxburg, U., Guyon, I., and Garnett, R.,
  editors, {\em Advances in Neural Information Processing Systems}, volume~29.
  Curran Associates, Inc.

\bibitem[Lu et~al., 2019]{lu2019scaling}
Lu, J., Lu, Y., and Nolen, J. (2019).
\newblock {Scaling limit of the Stein variational gradient descent: The mean
  field regime}.
\newblock {\em SIAM Journal on Mathematical Analysis}, 51(2):648--671.

\bibitem[Paulsen and Raghupathi, 2016]{Paulsen2016}
Paulsen, V.~I. and Raghupathi, M. (2016).
\newblock {\em {An Introduction to the Theory of Reproducing Kernel Hilbert
  Spaces}}.
\newblock Cambridge University Press.

\bibitem[Rezende and Mohamed, 2015]{rezende2015variational}
Rezende, D. and Mohamed, S. (2015).
\newblock Variational inference with normalizing flows.
\newblock In {\em International Conference on Machine Learning}, pages
  1530--1538. PMLR.

\bibitem[Scovel et~al., 2010]{Scovel2010}
Scovel, C., Hush, D., Steinwart, I., and Theiler, J. (2010).
\newblock {Radial kernels and their reproducing kernel Hilbert spaces}.
\newblock {\em Journal of Complexity}, 26(6):641--660.

\bibitem[Shi et~al., 2021]{shi2021sampling}
Shi, J., Liu, C., and Mackey, L. (2021).
\newblock {Sampling with Mirrored Stein Operators}.
\newblock {\em arXiv preprint arXiv:2106.12506}.

\bibitem[Sriperumbudur et~al., 2011]{sriperumbudur2011universality}
Sriperumbudur, B.~K., Fukumizu, K., and Lanckriet, G.~R. (2011).
\newblock {Universality, Characteristic Kernels and RKHS Embedding of
  Measures.}
\newblock {\em Journal of Machine Learning Research}, 12(7).

\bibitem[Staneva and Younes, 2017]{staneva2017learning}
Staneva, V. and Younes, L. (2017).
\newblock Learning shape trends: parameter estimation in diffusions on shape
  manifolds.
\newblock In {\em Proceedings of the IEEE Conference on Computer Vision and
  Pattern Recognition Workshops}, pages 38--46.

\bibitem[Sz{\'e}kely and Rizzo, 2013]{szekely2013energy}
Sz{\'e}kely, G.~J. and Rizzo, M.~L. (2013).
\newblock Energy statistics: A class of statistics based on distances.
\newblock {\em Journal of statistical planning and inference},
  143(8):1249--1272.

\bibitem[Wang and Liu, 2019]{wang2019nonlinear}
Wang, D. and Liu, Q. (2019).
\newblock {Nonlinear Stein Variational Gradient Descent for Learning
  Diversified Mixture Models}.
\newblock In Chaudhuri, K. and Salakhutdinov, R., editors, {\em Proceedings of
  the 36th International Conference on Machine Learning}, volume~97 of {\em
  Proceedings of Machine Learning Research}, pages 6576--6585. PMLR.

\bibitem[Wang et~al., 2018]{wang2018stein}
Wang, D., Zeng, Z., and Liu, Q. (2018).
\newblock {Stein Variational Message Passing for Continuous Graphical Models}.
\newblock In Dy, J.~G. and Krause, A., editors, {\em Proceedings of the 35th
  International Conference on Machine Learning, {ICML} 2018,
  Stockholmsm{\"{a}}ssan, Stockholm, Sweden, July 10-15, 2018}, volume~80 of
  {\em Proceedings of Machine Learning Research}, pages 5206--5214. {PMLR}.

\bibitem[Welling and Teh, 2011]{welling2011bayesian}
Welling, M. and Teh, Y.~W. (2011).
\newblock {Bayesian learning via stochastic gradient Langevin dynamics}.
\newblock In {\em Proceedings of the 28th international conference on machine
  learning (ICML-11)}, pages 681--688. Citeseer.

\bibitem[Yoon et~al., 2018]{yoon2018bayesian}
Yoon, J., Kim, T., Dia, O., Kim, S., Bengio, Y., and Ahn, S. (2018).
\newblock {Bayesian Model-Agnostic Meta-Learning}.
\newblock In Bengio, S., Wallach, H., Larochelle, H., Grauman, K.,
  Cesa-Bianchi, N., and Garnett, R., editors, {\em Advances in Neural
  Information Processing Systems}, volume~31. Curran Associates, Inc.

\bibitem[Zhuo et~al., 2018]{zhuo2018message}
Zhuo, J., Liu, C., Shi, J., Zhu, J., Chen, N., and Zhang, B. (2018).
\newblock {Message passing Stein variational gradient descent}.
\newblock In {\em International Conference on Machine Learning}, pages
  6018--6027. PMLR.

\end{thebibliography}


\begin{thebibliography}{7}
\providecommand{\natexlab}[1]{#1}
\providecommand{\url}[1]{\texttt{#1}}
\expandafter\ifx\csname urlstyle\endcsname\relax
  \providecommand{\doi}[1]{doi: #1}\else
  \providecommand{\doi}{doi: \begingroup \urlstyle{rm}\Url}\fi

\bibitem[Ba et~al.(2019)Ba, Erdogdu, Ghassemi, Suzuki, Sun, Wu, and
  Zhang]{ba2019towards}
Jimmy Ba, Murat~A Erdogdu, Marzyeh Ghassemi, Taiji Suzuki, Shengyang Sun, Denny
  Wu, and Tianzong Zhang.
\newblock {Towards Characterizing the High-dimensional Bias of Kernel-based
  Particle Inference Algorithms}.
\newblock \emph{AABI 2019, Symposium on Advances in Approximate Bayesian
  Inference, 2019}, 2019.

\bibitem[Boumal(2020)]{boumal2020introduction}
Nicolas Boumal.
\newblock An introduction to optimization on smooth manifolds.
\newblock \emph{Available online, May}, 2020.

\bibitem[Gong et~al.(2021)Gong, Li, and Hern{\'a}ndez-Lobato]{gong2021sliced}
Wenbo Gong, Yingzhen Li, and Jos{\'e}~Miguel Hern{\'a}ndez-Lobato.
\newblock {Sliced Kernelized Stein Discrepancy}.
\newblock In \emph{International Conference on Learning Representations}, 2021.
\newblock URL \url{https://openreview.net/forum?id=t0TaKv0Gx6Z}.

\bibitem[Liu and Wang(2016)]{Liu2016SVGD}
Qiang Liu and Dilin Wang.
\newblock {Stein Variational Gradient Descent: A General Purpose Bayesian
  Inference Algorithm}.
\newblock In D.~Lee, M.~Sugiyama, U.~Luxburg, I.~Guyon, and R.~Garnett,
  editors, \emph{Advances in Neural Information Processing Systems}, volume~29.
  Curran Associates, Inc., 2016.

\bibitem[Liu et~al.(2016)Liu, Lee, and Jordan]{Liu2016}
Qiang Liu, Jason Lee, and Michael Jordan.
\newblock {A Kernelized Stein Discrepancy for Goodness-of-fit Tests}.
\newblock In Maria~Florina Balcan and Kilian~Q. Weinberger, editors,
  \emph{Proceedings of The 33rd International Conference on Machine Learning},
  volume~48 of \emph{Proceedings of Machine Learning Research}, pages 276--284,
  New York, New York, USA, 20--22 Jun 2016. PMLR.

\bibitem[Townsend et~al.(2016)Townsend, Koep, and Weichwald]{Townsend2016}
James Townsend, Niklas Koep, and Sebastian Weichwald.
\newblock Pymanopt: A python toolbox for optimization on manifolds using
  automatic differentiation.
\newblock \emph{Journal of Machine Learning Research}, 17\penalty0
  (137):\penalty0 1--5, 2016.
\newblock URL \url{http://jmlr.org/papers/v17/16-177.html}.

\bibitem[Zhuo et~al.(2018)Zhuo, Liu, Shi, Zhu, Chen, and
  Zhang]{zhuo2018message}
Jingwei Zhuo, Chang Liu, Jiaxin Shi, Jun Zhu, Ning Chen, and Bo~Zhang.
\newblock {Message passing Stein variational gradient descent}.
\newblock In \emph{International Conference on Machine Learning}, pages
  6018--6027. PMLR, 2018.

\end{thebibliography}

%%%%%%%%%%%%%%%%%%%%%%%%%%%%%%%%%%%
%%%%%% SUPPLEMENT (OPTIONAL) %%%%%%
%%%%%%%%%%%%%%%%%%%%%%%%%%%%%%%%%%%

\clearpage
\appendix

% For one-column format, uncomment the following:
\onecolumn \makesupplementtitle
% For two-column format, uncomment the following:
%\twocolumn[ \makesupplementtitle ]

\section{FURTHER BACKGROUND MATERIALS ON THE GRASSMANN MANIFOLD}
\label{appendix: grassmann manifold}

We provide a brief overview of the Grassmann manifold and optimisation on the Grassmann manifold. We focus only on the essential components required to understand the proposed method; a thorough treatment of these topics is beyond the scope of the paper. For more details, we refer the interested readers to \citet{bendokat2020grassmann} regarding the theoretical perspectives of the Grassmann manifold, and \citet{boumal2020introduction, AbsMahSep2008} regarding optimisation on Grassmann and other manifolds.

\paragraph{The Stiefel and Grassmann Manifolds}
The \textit{Grassmann manifold} $\text{Gr}(d,m)$ \cite[Section 9]{boumal2020introduction} is the space of all $m$-dimensional linear subspaces in $\bbR^{d}$. Each element of $\text{Gr}(d,m)$ is an equivalence class of $d \times m$ matrices whose columns form an orthonormal basis for the same subspace. To this end, define the \emph{Stiefel manifold} $\text{St}(d,m):=\{A\in\mathbb{R}^{d\times m}: A^\intercal A = I_m \}$, where $I_{m}$ is the $m\times m$ identity matrix. That is, \text{St}(d,m) is the set of projectors of rank $m$ as defined in Section~\ref{sec:background}. The Grassmann manifold is built from the Stiefel manifold by identifying matrices whose columns span the same subspace. Formally, define an equivalence relationship $\sim$ as $A\sim B \Longleftrightarrow A = B C$ for any $C\in O(m)$, the group of orthogonal matrices in $\bbR^{m \times m}$. Gr($d,m$) is the quotient space 
\begin{talign*}
    \text{Gr}(d,m) 
    &= \text{St}(d,m)/\sim\:=\{[A]: A\in\text{St}(d,m)\}, \quad
    \textrm{where } \:[A]
    =\left\{B\in\text{St}(d,m): A\sim B\right\}.
\end{talign*}
The Grassmann manifold is an instance of \emph{quotient manifolds} on which we can define smooth functions; see \citet[Chapter 9]{boumal2020introduction} for a full treatment of quotient manifolds and \citet[Chapter 3]{boumal2020introduction} for the precise definition of smoothness of a manifold-valued function. This allows the use of gradient-type optimisation to find the critical points of a smooth function. A function $\calF$ defined on $\text{Gr}(d, m)$ differs from one defined on $\text{St}(d, m)$ in that the former must preserves invariance of the chosen projector; that is, $\calF([A])$ depends only on the subspace $[A] \in \text{Gr}(d, m)$ but not on $A$. An example is the matrix-valued kernel defined in Lemma~\ref{lemma: grassmann_kernel}: $A \mapsto k_A(x, y) = A A^\intercal k(A^\intercal x, A^\intercal y)$, $x, y \in \bbR^d$.

\paragraph{Tangent Spaces and the Riemannian Gradient}
To define the Riemannian gradient $\riemannGrad \mathcal{F}([A])$ of a smooth function $\mathcal{F}$ on $\text{Gr}(d,m)$, we endow $\text{Gr}(d,m)$ with the standard inner matrix product $g_{[A]}(\Delta, \tilde{\Delta}) \coloneqq \langle \Delta, \tilde{\Delta} \rangle_0 = \Tr(\Delta^\intercal \tilde{\Delta})$ restricted to $\Delta, \tilde{\Delta} \in \text{Gr}(d, m)$. The Riemannian gradient is then defined by the relationship
\begin{talign*}
  g_{[A]}(\riemannGrad \mathcal{F}([A]), \Delta) 
  &= \frac{d}{d\delta}\mathcal{F}(\textrm{R}_{[A]}(\delta\Delta))\Big|_{\delta= 0},
\end{talign*}
for $\Delta \in \mathcal{T}_{[A]}\text{Gr}(d,m) \coloneqq \{ \Delta \in \bbR^{d \times m}: A^\intercal \Delta = 0 \}$, and $\textrm{R}_{[A]}(\cdot)$ is a \emph{retraction} \cite[Definition 3.41]{boumal2020introduction} that send $[A]$ along a chosen direction while remaining on the Grassmann manifold. In particular, $\mathcal{T}_{[A]}\text{Gr}(d,m)$ is called the \emph{tangent space} at $[A]$ \cite[Definition 3.10]{boumal2020introduction}. For the Grassmann manifold, the Riemannian gradient is simply the orthogonal projection of the standard gradient $\nabla \calF([A])$ to the tangent space: $
\riemannGrad \calF([A]) =  \Pi_A \nabla \calF([A])$, where $\Pi_A = I_d - A A^\intercal$, and $[\nabla \calF([A])]_{ij} = \frac{\partial}{\partial A_{ij}} \calF([A]) $. 

\paragraph{Riemannian Gradient Descent}
To find a critical point of $\calF$ where its Riemannian gradient is zero, \emph{Riemannian gradient descent} (RGD), which generalised the standard gradient descent to a Riemannian manifold, can be used. In RGD, we start from $[A_0] \in \text{Gr}(d, m)$ and iterate $[A_{t + 1}] = \textrm{R}_{[A_t]}(- \delta \riemannGrad \calF([A_t]))$ for $t = 0, 1, \ldots$ and step size $\delta > 0$. Intuitively, $\riemannGrad \calF([A_t])$ finds the direction of steepest descent at $[A_t]$, and the retraction $\textrm{R}_{[A_t]}(\cdot)$ is the action of moving along $\text{Gr}(d, m)$ in that direction.

\section{PROOF of LEMMA \ref{lemma: grassmann_kernel}}
\label{appendix: proof of lemma grassmann_kernel}
\begin{proof}
% Let  $A,  B  \in  \mbox{St}(d,m)$.  We show that if $B  \in  [A]$, so both $A$ and $B$ belong to the same Grassmann equivalence class, then $\mbox{KSD}_A(P,Q) = \mbox{KSD}_B(P,Q)$. 
One can show (e.g.\ see \cite{bendokat2020grassmann}) that  $B \in [A]$ if and only if there exists an orthogonal matrix $C \in \mathbb{R}^{m\times m}$ such that $B = AC$.  Then, using the orthogonality of $C$ we obtain
\begin{talign*}
k_{B}(x,y)  = (AC) (AC)^\intercal k((AC)^\intercal x, (AC)^\intercal y)  &=  A C C^\intercal A^\intercal  k(C^\intercal A^\intercal x, C^\intercal A^\intercal y) \\
&=  A A^\intercal \Psi\left(\lVert C^\intercal (A^\intercal x - A^\intercal y)\rVert_2\right)\\
&= A A^\intercal \Psi\left(\lVert A^\intercal x - A^\intercal y\rVert_2\right) = k_A(x,y),
\end{talign*}
for all $x,y\in \mathbb{R}^d$.  

% The equality of the associated Grassman Kernel Stein Discrepancies follows immediately from the construction.  It follows that the definition of the Grassmann KSD is well defined as a function on $\text{St}(d,m)$.
\end{proof}

\section{PROOF of THEOREM \ref{thm:KSD_separates}}
The idea of the proof is to rewrite GKSD as a double integral type discrepancy, similar to \citet[Theorem 3.6]{Liu2016}, and the proof shall largely follow. We will then conclude that the discrepancy distinguishes $p,q$ given the assumptions on the kernel. 

The main difference to \citet[Theorem 3.6]{Liu2016} is that we shall be using a matrix-valued kernel formed from the projection $A$, as written in \eqref{eq:KSD_A}, rather than a standard scalar-valued kernel. This means we must introduce some extra notation. First, given a density $q$ on $\bbR^d$, we will be using the operator $\calA_{q}f(x) = s_{q}(x)^{\intercal}f(x) + \nabla\cdot f(x)$ for functions $f\colon\bbR^{d}\rightarrow\bbR^{d}$ and abuse notation so that for functions $F\colon\bbR^{d}\rightarrow\bbR^{d\times m}$ we have $\calA_{q}F(x) = (\calA_{q}F_{1}(x),\ldots,\calA_{q}F_{m}(x)) \in \bbR^{1\times m}$ where $F_{i}\colon\bbR^{d}\rightarrow\bbR^{d}$ is the $i$-th column of $F$. 

The next lemma is a vectorised version of \citet[Lemma 2.2, Lemma 2.3]{Liu2016} and the proof follows by applying those results elementwise to $F$. 

\begin{lemma}\label{lem:Steins_identity}
Under the assumptions on $q$ in Theorem \ref{thm:KSD_separates}, if $F\colon\bbR^{d}\rightarrow\bbR^{d\times m}$ has bounded entries then 
\begin{talign*}
    \bbE_{Q}\left[(s_{p}(x)-s_{q}(x))^{\intercal}F(x)\right] = \bbE_{Q}\left[\calA_{p}F(x)\right].
\end{talign*}
\end{lemma}
% \begin{proof}
% The bounded assumption along with the assumptions on $q$ means that each column of $F$ is in the Stein class of $q$, hence the result is a vectorised version of \cite[Lemma 2.2,Lemma 2.3]{Liu2016} which is also known as Stein's identity. \ad{Either explain what a Stein class is, or just state the property directly.} 
% \end{proof}

\begin{remark}
    The boundedness condition on $F$ is to ensure it lies in the \emph{Stein class} of $q$ \cite[Definition~2.1]{Liu2016}, since $q$ is assumed to be supported on all of $\bbR^d$, so that \citet[Lemma2.2, Lemma2.3]{Liu2016} can be applied. A similar condition on $F$ can be imposed when the support of $q$ is a compact subset $\mathcal{X}$ of $\bbR^d$. 
\end{remark}

Now we derive the analogy to \citet[Theorem 3.6]{Liu2016} using Lemma \ref{lem:Steins_identity}, the proof is largely the same as the proof of \citet[Theorem 3.6]{Liu2016}, which we include for clarity. 

\begin{lemma}\label{lem:KSD_quadratic}
Under the assumptions in Theorem \ref{thm:KSD_separates} the projected KSD can be expressed as follows 
\begin{talign}
    \emph{GKSD}(Q,P) = \sup_{[A]\in \emph{Gr}(d,m)}\bbE_{x, x' \sim Q}\left[(A^\intercal \delta_{p,q}(x))^{\intercal}A^\intercal \delta_{p,q}(x')k(A^\intercal x,A^\intercal x')\right],
    \label{eq:KSD_quadratic}
\end{talign}
where $\delta_{p,q}(x) = s_{p}(x)-s_{q}(x)$.
\end{lemma}
\begin{proof}
Define $v(x,x') = \calA_{p}(A k(A^\intercal \cdot,A^\intercal x'))(x)\in\bbR^{1\times m}$. Since $k$ is assumed to be bounded, we can apply Lemma \ref{lem:Steins_identity} with $F(x) = A k(A^\intercal x,A^\intercal x')$ to yield 
\begin{talign*}
    \bbE_{x, x' \sim Q}\left[(A^\intercal\delta_{p,q}(x))^{\intercal}A^\intercal \delta_{p,q}(x')k(A^\intercal x,A^\intercal x')\right] & = \bbE_{x, x' \sim Q}\left[v(x,x')A^\intercal \delta_{p,q}(x')\right]\\
    & = \bbE_{x, x' \sim Q}\left[\delta_{p,q}(x')^{\intercal}A v(x,x')^{\intercal}\right].
\end{talign*}

The assumption on the second order partial derivatives of the kernel and the proof of \citet[Theorem 3.6]{Liu2016} assure us that we can apply Lemma \ref{lem:Steins_identity} again with $F(x') = A v(x,x')^{\intercal}$ to get
\begin{talign*}
     \bbE_{x, x' \sim Q}\left[(A^\intercal \delta_{p,q}(x))^{\intercal}A^\intercal \delta_{p,q}(x')k(A^\intercal x,A^\intercal x')\right] = \bbE_{x, x' \sim Q}\left[\calA_{p}(A v(x,\cdot)^{\intercal})(x')\right].
\end{talign*}
Straightforward calculations along with $A^\intercal A=I_{m}$ show that
\begin{talign*}
    \calA_{p}(A v(x,\cdot)^{\intercal})(x') 
    & = (A^\intercal s_{p}(x))^{\intercal}(A^\intercal s_{p}(x'))k(A^\intercal x,A^\intercal x') + (A^\intercal s_{p}(x))^{\intercal} \nabla_{x_2}k(A^\intercal x,A^\intercal x') \\
    & + \nabla_{x_1}k(A^\intercal x,A^\intercal x')^\intercal A^\intercal s_p(x') + \text{Tr}(\nabla_{x_1,x_2}k(A^\intercal x,A^\intercal x')),
\end{talign*}
where $\nabla_{x_i}$ denotes the gradient of $k$ with respect to the $i^{th}$ argument, and $\nabla_{x_1, x_2} k$ is the matrix with $ij$-th entry $\partial^2 k / \partial x_i \partial x_j$. This completes the proof since it is the integrand of \eqref{eq:GKSD_vec}.
\end{proof}

\begin{proof}[Proof of Theorem \ref{thm:KSD_separates}]
    First, if $p=q$ then Lemma \ref{lem:KSD_quadratic} immediately tells us that the discrepancy is zero since $\delta_{p,q} = 0$. Now suppose the discrepancy is zero. Take any projector $A$, then $\|A^\intercal x-A^\intercal x'\|_2\leq \|x-x'\|_2$ since $A A^\intercal$ is a projection matrix. As $k$ is a radial kernel we know that $\Psi$ is monotonically decreasing \citep[Theorem 1.1]{Scovel2010}. Therefore $k(Ax,Ax')\geq k_{d}(x,x')$, where $k_{d} = \Psi(\|x-x'\|_2)$. By Lemma \ref{lem:KSD_quadratic},
    \begin{talign}
        0 = \text{GKSD}(Q,P) \geq \sup_{[A]\in\text{Gr}(d,m)}\bbE_{x, x' \sim Q}\left[(A^\intercal \delta_{p,q}(x))^{\intercal}A^\intercal \delta_{p,q}(x')k_{d}(x,x')\right].
        \label{eq:GKSD_upper_bound}
    \end{talign}
    Now choose $A$ as the matrix formed by stacking the first $m$ canonical basis elements. Since $k$ is characteristic and radial, we know that $\Psi(\|x-x'\|_2) = L_{\nu}(\|x-x'\|_2)$ where $L_{\nu}$ is the Laplace transform of a measure $\nu$ on $[0,\infty)$ with full support \citep[Proposition 5]{sriperumbudur2011universality}. Therefore $k_{d}(x,x') = \Psi(\|x-x'\|_2) = L_{\nu}(\|x-x'\|_2)$ and $k_{d}$ is also characteristic and hence integrally strictly positive definite.

    Using this we can conclude from \eqref{eq:GKSD_upper_bound} that $A^\intercal \delta_{p,q} = 0$, which implies $A^\intercal s_{q} = A^\intercal s_{q}$ so the first $m$ entries of the two score functions are the same. Repeat this argument by setting $A$ to be the next $m$ canonical basis elements until we have checked all canonical basis elements. We then deduce that the score functions are equal in all coordinates, which implies $p=q$ as required. 
\end{proof}

\section{PROOF of PROPOSITION \ref{prop: optimal projection}}
\label{appendix: lemma and proof of proposition 3}
We prove that, provided the densities of the distributions $Q, P$ differ only in a subspace identified by a projector  $A_0$ of rank $m$, the Riemannian gradient $\riemannGrad \alpha([A])$ is zero at $A = A_0$. That is, the optimal subspace $[A_0]$ is indeed a solution sought by the objective 
\begin{talign*}
    \sup_{[A]\in\text{Gr}(d,m)} \textmd{KSD}_A(Q, P) ,
\end{talign*}
where $\textmd{KSD}_A(Q, P)$ is given by (\ref{eq:KSD_double_integral}). 

In the rest of this section, we first state an assumption on the form of the kernel in \ref{appendix: extra assumption on the kernel} that is mild but can greatly simplify the proof. In \ref{appendix: decomposition of q and p}, we then provide intuition on the assumption of the decomposibility of the densities of $Q$ and $P$ and give a concrete example. The proof is presented in \ref{appendix: proof of proposition 3}.

\subsection{Assumption on the Kernel}
\label{appendix: extra assumption on the kernel}
In Theorem~\ref{thm:KSD_separates}, we have assumed the kernel takes the form $k(u, v) = \Psi(\| u - v \|_2)$, i.e.\ it is a \emph{radial kernel} defined in \citet{sriperumbudur2011universality}. To simplify the form of $\riemannGrad \alpha([A])$ in our proof, we introduce the following reformulation of $\Psi$ and impose a smoothness condition.

% \begin{assumption}[Smoothness of kernel]
%     The kernel has the form $k(u, v) = \Phi(\| u - v\|_2^2)$, where $\Phi: \bbR \to \bbR$ is continuously differentiable.
%     \label{assumption: smoothness of kernel}
% \end{assumption}
\begin{assumption}[Smoothness of kernel]
    There exists a continuously differentiable function $\Phi: \bbR \to \bbR$ for which $\Psi(s) = \Phi(s^2)$ for all $s \geq 0$. That is, the kernel has the form $k(u, v) = \Phi(\| u - v\|_2^2)$.
    \label{assumption: smoothness of kernel}
\end{assumption}
% Compared with the radial basis form $k(u, v) = \Psi( \| u - v\|_2)$ in Theorem \ref{thm:KSD_separates}, we introduced the reformulation $\Phi( s ) = \Psi(\sqrt{s})$, $s > 0$, so that it gives rise to a simpler form of the matrix gradient in the following lemma. Furthermore, the extra continuous differentiability assumption on $\Phi$ is for convenience only. In particular, Assumption~\ref{assumption: smoothness of kernel} holds for Gaussian RBF and Inverse Multiquadric kernels.
\begin{remark}
    The continuous differentiability condition on $\Phi$ is for convenience only and is mild. In particular, both Gaussian RBF and Inverse Multiquadric kernels satisfy Assumption~\ref{assumption: smoothness of kernel}.
\end{remark}

\subsection{Assumption on the Densities}
\label{appendix: decomposition of q and p}
Intuitively, the decomposition (\ref{eq: prop optimal projection, decomposition 1}) and (\ref{eq: prop optimal projection, decomposition 2}) states that the the candidate density $q$ and the target $p$ are distinct only in the subspace $[A_0]$. It can be viewed as a type of sparsity assumption. Both the multimodal mixture and the X-shaped mixture examples in the experiments satisfy this assumption, where it is easy to check that such functions $q^m, p^m, \xi$ can be found with $[A_0] = \{ a e_1 + b e_2: a, b \in \bbR \}$ being the subspace spanned by the first two canonical basis vectors $e_1, e_2$. See Appendix \ref{appendix:multimodal} and \ref{appendix:xshaped} for the precise definition of the target densities in these two examples.

More generally, this assumption holds when $q(x) $ is a prior, and $p(x) \propto q(x) f(y | P_0 x)$ is the posterior induced by observational data $y$ and a likelihood function $f(y | A_0A_0^\intercal x)$ that depends on $x$ only through its projection onto the subspace $[A_0]$. When the columns of $A_0$ are canonical basis vectors, this means the data-generating process only depends on a subset of the parameters $x$. A similar setting is also considered by \citet{chen2020projected}, where the authors argue that such a likelihood function is a reasonable approximation when the variation of the likelihood is negligible outside of some eigenspace of a gradient information matrix.

\subsection{Proof}
\label{appendix: proof of proposition 3}
We begin by deriving an expression for the Euclidean and Riemannian gradients of $\alpha([A])$ with respect to a projector $A$.

\begin{lemma}
    Suppose that Assumption \ref{assumption: smoothness of kernel} holds, and that $p, q$ satisfy the same assumptions in Theorem~\ref{thm:KSD_separates}. Then, given a projector $A$, the Euclidean gradient of $\alpha([A])$ at $A$ is
    \begin{alignat}{2}
        \nabla \alpha([A])
        = 
        & 2\bbE_{x, x' \sim Q}[ && k(A^\intercal x, A^\intercal x') 
        \delta_{p, q}(x) \delta_{p, q}(x')^\intercal A \nonumber \\
        & &&+ \Phi'\left(\| A^\intercal x - A^\intercal x'\|_2^2 \right) \delta_{p, q}(x)^\intercal A A^\intercal \delta_{p, q}(x') (x - x')(x - x')^\intercal A
         ],
       \label{eq: euclidean grad of alpha}
    \end{alignat}
    where the expectation is taken over i.i.d.\ copies $x, x' \sim Q$, $\Phi'$ is the first derivative of $\Phi$, and $[\nabla \alpha([A])]_{ij} = \frac{\partial}{\partial A_{ij}} \alpha([A])$. Moreover, letting $\Pi_A \coloneqq I_d - A A^\intercal$, the Riemannian gradient of $\alpha([A])$ at $A$ is 
    \begin{alignat}{2}
        \riemannGrad \alpha([A])
        &= \Pi_A\nabla \alpha([A])  \nonumber \\
        &=
        2\Pi_A \bbE_{x, x' \sim Q}[ && k(A^\intercal x, A^\intercal x') 
        \delta_{p, q}(x') \delta_{p, q}(x)^\intercal A 
        \nonumber \\ 
        & && +\Phi'\left(\| A^\intercal x - A^\intercal x'\|_2^2 \right) \delta_{p,q}(x')^\intercal A A^\intercal \delta_{p,q}(x) (x - x')(x - x')^\intercal A  ] .
        \label{eq: riemannian grad of alpha}
    \end{alignat}
    \label{lemma: gradients of alpha}
\end{lemma}

\begin{proof}[Proof of Lemma~\ref{lemma: gradients of alpha}]
    Letting $G$ be the right hand side of (\ref{eq: euclidean grad of alpha}). We will show that, for any $t \in \bbR$ and projectors $A, B$ of rank $m$, 
    \begin{align*}
        \alpha([A + tB]) - \alpha([A])
        = t \cdot \Tr(G^\intercal B ) + \calO(t^2).
    \end{align*}
    Let $w \coloneqq x - x'$ so that $k(A^\intercal x, A^\intercal x') = \Phi( \| A^\intercal w\|_2^2 )$ for a projector $A$. Starting with $\alpha([A])$ in the form of (\ref{eq:KSD_quadratic}), 
    \begin{align}
        \alpha([A + tB]) - \alpha([A])
        &= \bbE_{x, x' \sim Q} \left[ \delta_{p, q}(x)^\intercal (A + tB) (A + tB)^\intercal \delta_{p, q}(x') \Phi\left( \left\| (A + tB)^\intercal w \right\|_2^2 \right)  \right] \nonumber \\
        & \quad - \bbE_{x, x' \sim Q} \left[ \delta_{p, q}(x)^\intercal A A^\intercal \delta_{p, q}(x') \Phi\left( \left\| A^\intercal w \right\|_2^2 \right)  \right] \nonumber \\
        &= \bbE_{x, x' \sim Q} \left[ \delta_{p, q}(x)^\intercal A A^\intercal \delta_{p, q}(x') 
        \left( \Phi\left( \left\|(A + tB)^\intercal w \right\|_2^2 \right) - \Phi\left( \left\|A^\intercal w \right\|_2^2 \right) \right)  \right] \nonumber \\
        & \quad + 2t \bbE_{x, x' \sim Q} \left[ \delta_{p, q}(x)^\intercal A B^\intercal \delta_{p, q}(x') \Phi\left( \left\|(A + tB)^\intercal w \right\|_2^2 \right)  \right] \nonumber \\
        & \quad + \calO(t^2) .
        \label{eq: taylor exp of alpha intermediate}
    \end{align}
    We now seek a linear approximation of the function $A \mapsto \Phi( \| A^\intercal w \|_2^2 )$, where $A$ is a projector of rank $m$. Since $\Phi$ is continuously differentiable by assumption, it admits a first-order Taylor expansion at $\|A^\intercal w\|_2^2$, which gives
    \begin{align}
        &\Phi\left( \left\| (A + tB)^\intercal w \right\|_2^2 \right) - \Phi\left( \left\| A^\intercal w \right\|_2^2 \right) \\
        &= \Phi\left( \left\| A^\intercal w \right\|_2^2 + t \left( 2w^\intercal A B^\intercal w + t\left\| B^\intercal w \right\|_2^2 \right) \right) - \Phi\left( \left\| A^\intercal w \right\|_2^2 \right) 
        \nonumber \\
        &= \Phi\left( \left\| A^\intercal w \right\|_2^2 \right) 
        + t \Phi'\left( \left\| A^\intercal w \right\|_2^2 \right) \left( 2w^\intercal A B^\intercal w + t \left\| B^\intercal w \right\|_2^2 \right) - \Phi\left( \left\| A^\intercal w \right\|_2^2 \right) + \calO(t^2)
        \nonumber \\
        &= t \Phi'\left( \left\| A^\intercal w \right\|_2^2 \right) \left( 2w^\intercal A B^\intercal w \right) + \calO(t^2).
        \label{eq: taylor expansion of kernel}
    \end{align}
    
    The rest of the proof for the form of the Euclidean gradient follows by substituting (\ref{eq: taylor expansion of kernel}) into (\ref{eq: taylor exp of alpha intermediate}) and using properties of the trace operator:
    \begin{align*}
        \alpha([A + tB]) - \alpha([A])
        &= 2t \bbE_{x, x' \sim Q} \left[ \Phi'\left(\left\| A^\intercal w \right\|_2^2 \right) \delta_{p, q}(x)^\intercal A A^\intercal \delta_{p, q}(x') w^\intercal A B^\intercal w \right] \\
        & \quad + 2t \bbE_{x, x' \sim Q}\left[ \delta_{p, q}(x)^\intercal A B^\intercal \delta_{p, q}(x')  \Phi\left( \left\| A^\intercal w \right\|_2^2 \right) \right] + \calO(t^2) \\
        &= 2t \bbE_{x, x' \sim Q} \left[ \Phi'\left(\left\| A^\intercal w \right\|_2^2 \right) \delta_{p, q}(x)^\intercal A A^\intercal \delta_{p, q}(x') \Tr \left( w w^\intercal A B^\intercal \right) \right] \\ 
        & \quad + 2t \bbE_{x, x' \sim Q} \left[ \Phi\left( \left\| A^\intercal w \right\|_2^2 \right) \Tr \left(\delta_{p, q}(x') \delta_{p, q}(x)^\intercal A B^\intercal \right) \right]  + \calO(t^2)\\
        &= t \cdot \Tr \left( 2 \bbE_{x, x' \sim Q} \left[ \Phi'\left(\left\| A^\intercal w \right\|_2^2 \right) \delta_{p, q}(x)^\intercal A A^\intercal \delta_{p, q}(x')  w w^\intercal A B^\intercal \right]\right) \\ 
        & \quad + t \cdot \Tr \left( 2 \bbE_{x, x' \sim Q} \left[ \Phi\left( \left\| A^\intercal w \right\|_2^2 \right) \delta_{p, q}(x') \delta_{p, q}(x)^\intercal A B^\intercal \right]\right)  + \calO(t^2)\\
        &= t \cdot \Tr ( G^\intercal B ) + \calO(t^2),
    \end{align*}
    where in the last equality we have substituted the definition of $G$ and used the fact $\Tr(G B^\intercal) = \Tr(G^\intercal B)$. Finally, it can be shown that $\text{Gr}(d, m)$ is a Riemannian submanifold of the Euclidean manifold, so the Riemannian gradient is the orthogonal projection of the Euclidean gradient to the tangent spaces (see e.g.\ \citet[Proposition 3.53]{boumal2020introduction}). That is, $\riemannGrad \alpha([A]) =  (I_d - A A^\intercal)\nabla \alpha([A]))$, so (\ref{eq: riemannian grad of alpha}) follows.
\end{proof}

\begin{proof}[Proof of Proposition \ref{prop: optimal projection}]
    Decomposing (\ref{eq: riemannian grad of alpha}) into two terms, we have
    \begin{align}
        \riemannGrad \alpha([A])
        &=
        2\Pi_A \bbE_{x, x' \sim Q}[ k(A^\intercal x, A^\intercal x') 
        \delta_{p, q}(x') \delta_{p, q}(x)^\intercal A] \nonumber \\ 
        & \quad + 2\Pi_A \bbE_{x, x' \sim Q}[ \Phi'\left(\| A^\intercal x - A^\intercal x'\|_2^2 \right) \delta_{p,q}(x')^\intercal A A^\intercal \delta_{p,q}(x) (x - x')(x - x')^\intercal A ],
        \label{eq: riemannian grad of alpha decomposed}
    \end{align}
    where $\Pi_A = I_d - A A^\intercal$. We will show that both terms equal to zero at $A = A_0$. 
    Firstly, define $\tilde{\xi}_{\Pi_{A_0}} (x) = \xi( \Pi_{A_0} x)$ and $P_0 = A_0 A_0^\intercal$. The score function of the candidate density is
    \begin{align*}
        \nabla_x \log q(x)
        &= \nabla_x \log (q^m \circ P_0)(x) + \nabla_x \log \tilde{\xi}_{\Pi_{A_0}} (x) \\
        &= P_0 \nabla_{x} \log q^m (P_0 x) + \nabla_x \log \tilde{\xi}_{\Pi_{A_0}} (x), 
    \end{align*}
    where $\nabla_x \log q^m(P_0 x)$ denotes the gradient of $\log q^m$ evaluated at $P_0 x$, and where the second line follows from the chain rule and the symmetry of $P_0$. Similarly, the score function of the target density takes the form
    % \begin{align*}
    %     \nabla_x \log p(x)
    %     = P_0 \nabla_{x} \log q^m(P_0 x)
    %     + \nabla_x \log \xi(x; P_0).
    % \end{align*}
    \begin{align*}
        \nabla_x \log p(x)
        = P_0 \nabla_{x} \log q^m(P_0 x)
        + \nabla_x \log \xi(x).
    \end{align*}
    Taking the difference,
    \begin{align*}
        \delta_{p, q}(x)
        = P_0 ( \nabla_{x} \log p^m(P_0x) - \nabla_{x} \log q^m(P_0 x)).
    \end{align*}
    Since $\Pi_{A_0} P_0 = 0$, we conclude that $\Pi_{A_0} \delta_{p, q}(x) = 0$, and the first term of (\ref{eq: riemannian grad of alpha decomposed}) is zero.
    
    For the second term, we define $a \coloneqq \delta_{p, q}^\intercal(x') A_0 A_0^\intercal \delta_{p, q}(x)$ and $b \coloneqq \Phi'(\| A^\intercal x - A^\intercal x'\|_2^2)$. Since both the terms $ \delta_{p, q}(x)$ and $A_0^\intercal x - A_0^\intercal x' = A_0^\intercal (P_0 x - P_0 x')$ are deterministic given $(P_0 x, P_0 x')$, so are $a$ and $b$. We further introduce the notations $w \coloneqq P_0 x, w' \coloneqq P_0 x'$ and $v \coloneqq \Pi_{A_0} x, v' \coloneqq \Pi_{A_0} x'$. Since $q^m, p^m$ and $\xi$ are unnormalised densities over $\bbR^d$, we can invoke the tower rule to simplify the second term as
    \begin{align*}
        \Pi_{A_0} \bbE_{x, x' \sim Q}\left[
        a b (x - x') (x - x')^\intercal A_0
        \right]
        &= \bbE_{x, x' \sim Q}\left[
        a b ( \Pi_{A_0} x - \Pi_{A_0} x') (P_0 x - P_0 x')^\intercal A_0
        \right] \\
        &= \bbE_{w, w' \sim q^m} \left[ \bbE_{v, v' \sim q^\perp} \left[
        a b (v - v') (w - w')^\intercal A_0
        \ | \ w, w' \right] \right] \\
        &= \bbE_{w, w' \sim q^m} \left[ a b \bbE_{v, v' \sim q^\perp} \left[
        v - v' \ | \ w, w' \right] 
        (w - w')^\intercal A_0 \right] ,
    \end{align*}
    where the first line follows from the fact that $A_0 = A_0 A_0^\intercal A_0 = P_0 A_0$. Since $x, x'$ are i.i.d.\ copies, we have $\bbE_{v, v' \sim q^\perp}[ v - v' | w, w' ] = \bbE_{q^\perp}[ v \ |\ w] - \bbE_{q^\perp}[ v' \ |\  w' ] = 0$, so the second term of (\ref{eq: riemannian grad of alpha decomposed}) is also zero. This shows that $\riemannGrad \alpha([A]) = 0$ at $A = A_0$, as required.
\end{proof}

\section{EXPERIMENTAL DETAILS}

\label{appendix: experimental details}
In this section, we provide implementation details for the experiments in Section~\ref{sec: experiments}.

\subsection{Setups}
\paragraph{Learning Rates}
The best step sizes of the gradient method for the particle update in all methods, and for the projector or slice update in GSVGD and S-SVGD, are selected over a grid of values from $10^{-4}$ to 1 such that they give the minimal energy distance between the particle estimation and the ground truth. This is done for the 50-dimensional multimodal example only, and the same set of learning rates is then used for all experiments. This is done for all methods to ensure a fair comparison. 

\paragraph{Temperature $T$}
For the temperature parameter $T$ in (\ref{eq: discretized projector update}), we adopt an annealing scheme where we start with $T = T_0$ and gradually increment it to a large value, $T_{\textmd{large}}$. Intuitively, setting $T$ small would lead to faster convergence, but hinders exploration of the maxima of $\alpha_t([A]) = \mathrm{KSD}_A(Q_t, P)$. On the other hand, a large $T$ means the system (\ref{eq:tangent_sde}) is effectively sampling projectors at random from the Grassmann manifold, thus allowing better exploration of the modes at the expense of a slower convergence. The annealing scheme allows trade-offs between such exploration and exploitation. Starting from $T_0$, at each iteration $t$, we multiply $T$ by 10 if the change $| \gamma_t - \gamma_{t-1} |$ in the \emph{particle-averaged magnitude} \citep{zhuo2018message} of the update function, defined as $\gamma_t \coloneqq \frac{1}{N} \sum_{i = 1}^N \| \sum_{l = 1}^M \widehat{\phi}_{A_{t, l}}(x_i^t) \|_\infty$, is less than a pre-specified threshold, where $N$ is the sample size, $M$ is the number of projectors, and the update function $\widehat{\phi}_{A_{t, l}}$ is defined in (\ref{eq: gsvgd update}). In the experiments, we used $T_0 = 10^{-4}$ and $T_{\textmd{large}} = 10^{6}$, and set the threshold to be $10^{-4} M$. This choice of the threshold is motivated by the dependence of the GSVGD particle update on the number of projectors.

\paragraph{Initialisation of Projectors}
The projectors in GSVGD are initialized with matrices formed by one-hot vectors. This is similar to the setup in S-SVGD, where the slices are initialized to be the canonical basis elements. When more than one projector is used in GSVGD, the projection dimensions, $m$, are kept fixed for all projectors for simplicity. The number of projectors, $M$, is chosen such that the projection dimensions sum up to $d$ and is capped at 20, where $d$ is the dimension of the problem. That is, $M = \min(20, \lfloor d/m \rfloor)$, where $\lfloor a \rfloor$ denotes the largest integer smaller than or equal to $a$

\paragraph{Other Setups}
For SVGD, we follow exactly the same setups as in \citet{Liu2016SVGD}. For S-SVGD, we follow mostly the same setups in Algorithm~2 and Appendix G of \citet{gong2021sliced}, except we updated the slices at every step instead of only when the particles have moved a sufficiently far distance from the previous step, as suggested in their paper. As noted in their paper, this trick is to prevent over-fitting of the slices to small samples, which is unlikely an issue in our experiments due to the large sample size we chose. Also, we remark that, for GSVGD, the coupled ODE-SDE system introduced in Section~\ref{sec:GSVGD} suggests a principled way to balance the evolution of the particles and projectors in GSVGD by varying the temperature $T$.

We ran each experiment with 20 random seeds, except Bayesian logistic regression which was repeated 10 times. The mean and $95\%$ confidence intervals are reported in all figures. A No-U-Turn Sampler (NUTS) \citep{hoffman2014no} is used as a gold standard in the conditioned diffusion and Bayesian logistic regression experiments. Each method ran for $2000$ iterations, which we found to be sufficient for convergence (see Figure~\ref{fig: appendix convergence vs epochs}).

\subsection{Multimodal Mixture}
\label{appendix:multimodal}
The mean vectors of the multimodal mixture are defined as $\mu_k = (\sqrt{5} \cos(2k\pi / K + \pi/4), \sqrt{5} \sin(2k\pi / K + \pi/4), 0, \ldots, 0 )^\intercal \in \bbR^d$, for $k = 1, \ldots 4$; see Figure~\ref{fig: multimodal particles 50d}. That is, the target in the first two dimensions is a mixture of 4 Gaussian distributions with modes equally spaced on a circle of radius $\sqrt{5}$, whilst in the other coordinates it is a standard Gaussian.

\subsection{X-Shaped Mixture}
\label{appendix:xshaped}
% \label{sec: x-shaped mixture details}
The covariance matrices in the X-shaped mixture example have the following correlated block diagonal structure:
\begin{talign*}
    \Sigma_1 = 
    \begin{pmatrix}
        \begin{matrix}
            1 & \delta \\
            \delta & 1
        \end{matrix}
        & \mathbf{0} \\
        \mathbf{0} & I_{d - 2}
    \end{pmatrix}
    ,
    \qquad
    \Sigma_2 = 
    \begin{pmatrix}
        \begin{matrix}
            1 & -\delta \\
            -\delta & 1
        \end{matrix}
        & \mathbf{0} \\
        \mathbf{0} & I_{d - 2}
    \end{pmatrix},
\end{talign*}
where $\delta = 0.95$ controls the correlation between the first two dimensions.

\subsection{Bayesian Logistic Regression}
We follow the same setting as \citet{Liu2016SVGD} by placing a Gaussian prior $p_0(w | \alpha) = \calN(w; 0, \alpha^{-1})$ on the regression weights $w$ with $p_0(\alpha) = \text{Gamma}(\alpha; 1, 0.01)$. We are interested in the posterior $p(x | D)$ of $x = [w, \alpha]$. The model was tested on a subset of the Covertype dataset. The original dataset consists of 581,012 data points and 54 features with binary labels, which is too large for NUTS. As a result, a subset of 1,000 randomly selected data points was used, although other methods that are more scalable to data size, such as the stochastic gradient Langevin dynamics (SGLD) of \citet{welling2011bayesian} or the HMC-ECS of \citet{dang2019hamiltonian}, could have been used to allow inference on the entire dataset.

To evaluate the results, we used both the energy distance and the covariance estimation error $\| \hat{\Sigma} - \Sigma\|_2$, where $\|A\|_2 = \sqrt{\sum_{i, j = 1}^d A^2_{ij}}$ is the Frobenius norm of a matrix $A \in \bbR^{d \times d}$, and $\hat{\Sigma}, \Sigma$ are respectively the sample covariance matrices of the particle estimation and of a HMC run treated as the ground truth. The covariance estimation error was used because we found that the energy distance is not sensitive enough to differences in the second moments. From Figure~\ref{fig:blr:covariance_estimation}, we observe an overestimated covariance between distinct variables compared with the HMC reference, as shown by the large values (in absolute term) of the off-diagonal entries. On the other hand, GSVGD projecting onto $m = 10$ dimensional subspaces showed greater alignment, despite having a similar energy distance as S-SVGD (Figure~\ref{fig:blr:energy distance}). Again, we emphasise that the performance of GSVGD deteriorates when the projection dimension is extreme, e.g.\ when $m$ equals 1 or the full dimension of the problem ($m = 55$).

\section{SUPPLEMENTARY FIGURES}
In this section, we include supplementary figures for the experiments in Section~\ref{sec: experiments} as well as some ablation studies.

\begin{figure}[t!]
    \centering
    \subfigure[Multivariate Gaussian]
    {
        \label{fig: gaussian energy vs epochs}
        \includegraphics[width=0.22\textwidth]{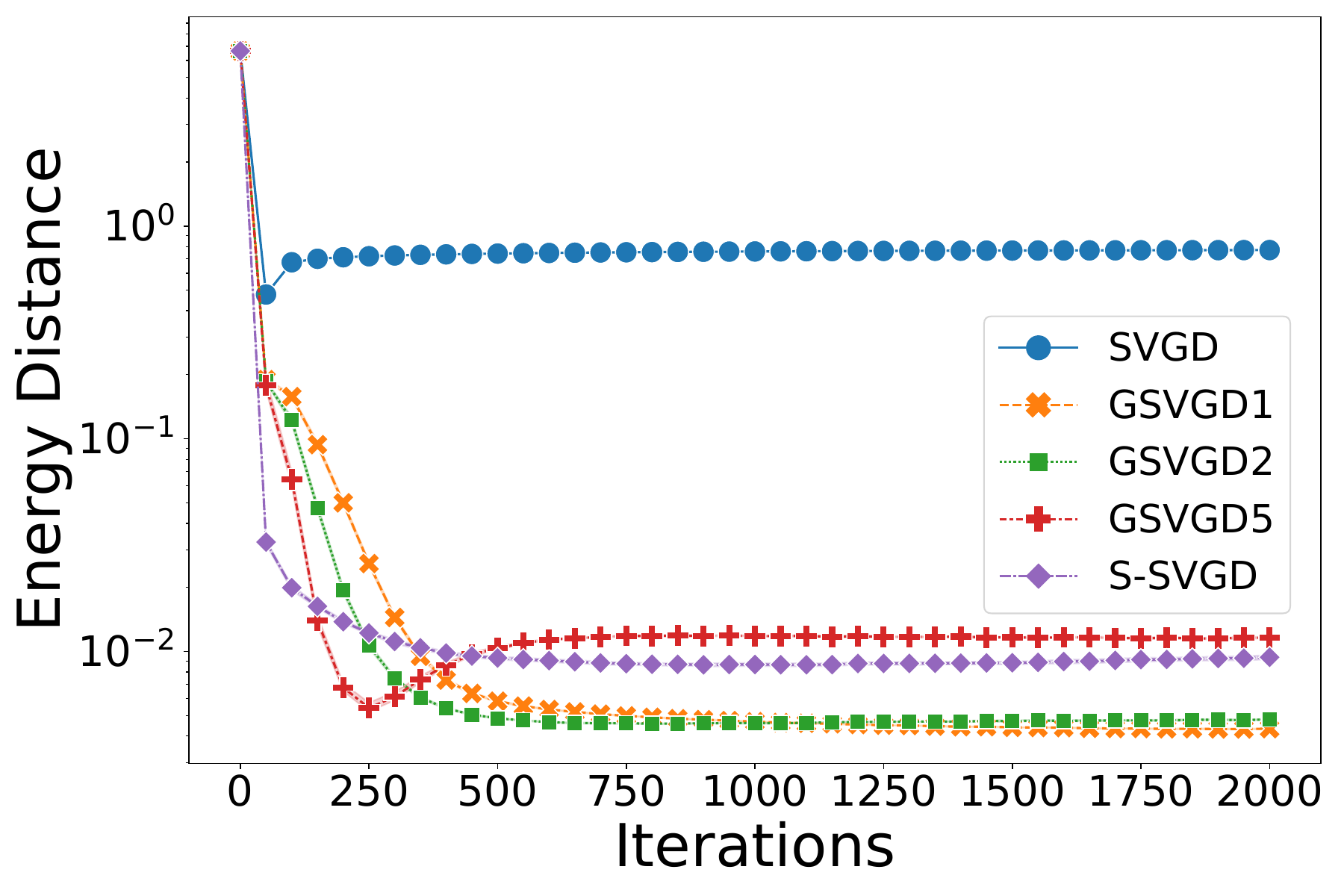}
    }
    \subfigure[Multimodal mixture]
    {
        \label{fig: multimodal energy vs epochs}
        \includegraphics[width=0.22\textwidth]{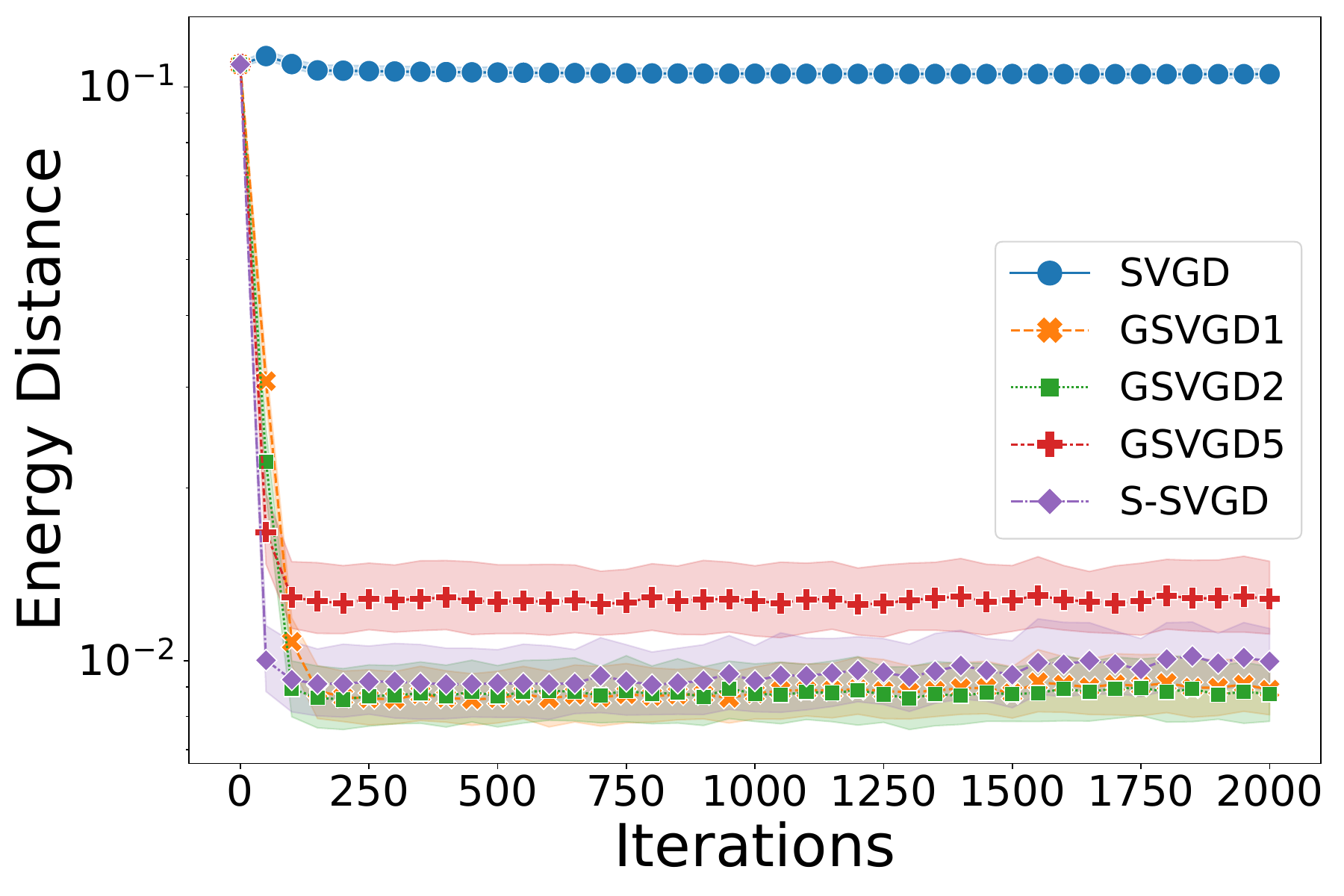}
    }
    \subfigure[X-shaped mixture]
    {
        \label{fig: xshaped energy vs epochs}
        \includegraphics[width=0.22\textwidth]{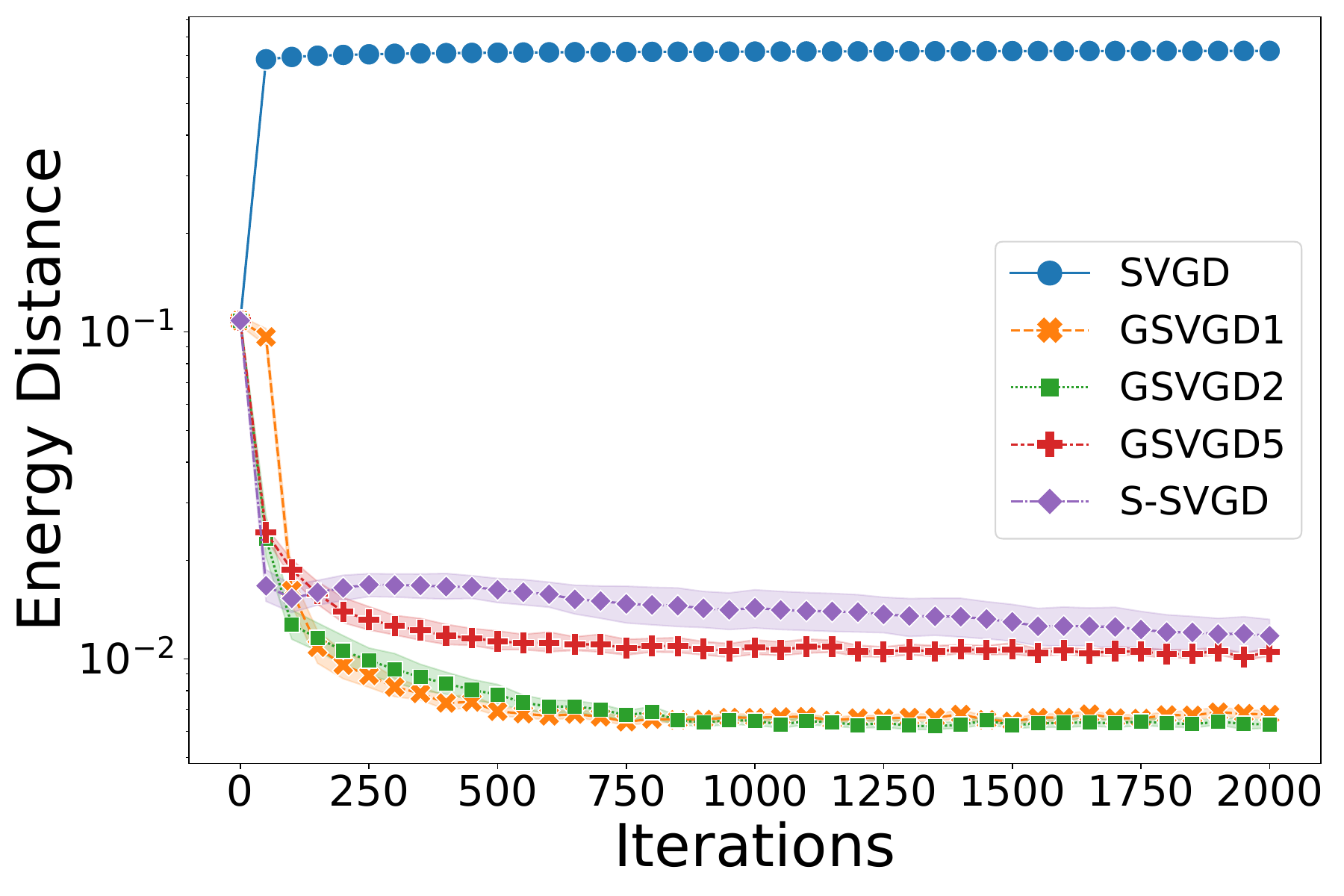}
    }
    % \subfigure[Multivariate Gaussian]
    % {
    %     \label{fig: gaussian nparticles}
    %     \includegraphics[width=0.25\textwidth]{figs/gaussian/var_nparticles.pdf}
    % }
    \caption{Convergence of particles in the (a) multivariate Gaussian, (b) multimodal mixture and (c) X-shaped experiments. The dimensionality is 50 in all cases.}
    \label{fig: appendix convergence vs epochs}
\end{figure}

\begin{figure}[t!]
\centering
\begin{minipage}{.45\textwidth}
    \centering
    \includegraphics[width=0.8\textwidth]{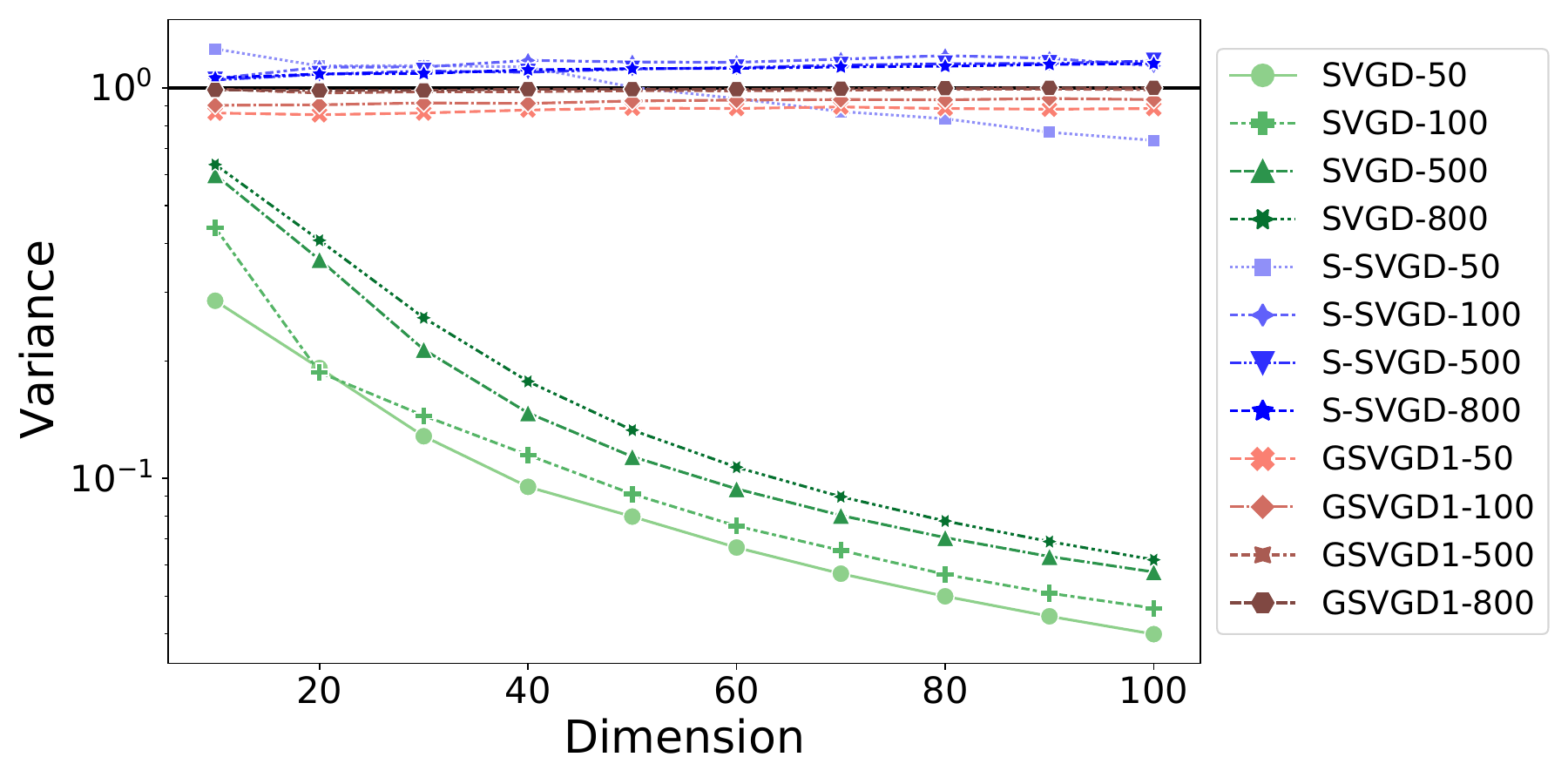}
    \caption{Variance estimates for different sample sizes, where the true value is shown by the black solid line. SVGD-50 means 50 particles are used to estimate the variance.}
    \label{fig: gaussian nparticles}
\end{minipage}
\hfill
\begin{minipage}{.45\textwidth}
    \centering
    \includegraphics[width=0.7\textwidth]{figs/diffusion/solution_effdim.pdf}
    \caption{Conditioned Diffusion Process: Energy distance between HMC and GSVGD solutions against different projection dimensions. SVGD and S-SVGD are also included for comparison.}
    \label{fig: diffusion energy}
\end{minipage}
\end{figure}

\subsection{Synthetic Experiments}
Figure~\ref{fig: appendix convergence vs epochs} supplements Figure~\ref{fig: summary metrics}, and shows the convergence of SGVD, S-SVGD and GSVGD to the target distribution in the three synthetic experiments in Section~\ref{sec: experiments}. We can see that the smallest energy distance between the particle estimation and the ground truth is achieved by GSVGD1 and GSVGD2, followed by S-SVGD and GSVGD5. GSVGD with any projection dimension has comparable per-iteration convergence rate as S-SVGD except for the multivariate Gaussian target. We believe this is because of the extra noise introduced via the coupled ODE-SDE system described in Section~\ref{sec:GSVGD}, which added unnecessary complexities for this simple, single-mode target. 
% Figure~\ref{fig: gaussian energy vs epochs} suggests that GSVGD converged slower than S-SVGD. We believe this is due to the initial choice of $T$ in the annealing scheme, which for the multivariate Gaussian example is set to $T_0 = 10^{-4}$. Setting it to a larger value should allow faster convergence. However, we emphasise that GSVGD1 and GSVGD2 are still able to converge to a lower value than S-SVGD.

\subsection{Conditioned Diffusion Process}
Figure~\ref{fig: diffusion energy} plots the energy distance between the particle estimation of each method and a HMC sampler against different projection dimensions, $m$. We see that there exists a feasible region of projection dimensions where GSVGD achieves a greater degree of agreement with the reference HMC, compared with SVGD and S-SVGD. This is expected given the problem setup: the forcing term is a Brownian motion with covariance $C(t, t') = \min(t, t')$, and hence is correlated between time points. This correlation means the solution state $u$ will admit a low-dimensional structure. 

Although in this case an appropriate choice of projection dimension could be inferred from the problem setup, selecting the best projection dimension for a general problem is an open question. 
Nonetheless, the shape of the curves in Figure~\ref{fig: diffusion energy} suggests that a feasible region should not cover the extreme ends.

\subsection{Ablation Studies}

\paragraph{Sample Size}
To evaluate the influence of sample size on the performance of GSVGD, we show the dimension-averaged variance estimates for the multivariate Gaussian target. We follow the same experimental setup in Section~\ref{sec: experiments}: for each of $m = 1, 2, 5$, we run each method for 2000 iterations and compute the dimension-averaged marginal variance of the particle estimation of the last iteration. Figure~\ref{fig: gaussian nparticles} shows the estimates for GSVGD1, SVGD and S-SVGD when the target is 50-dimensional. GSVGD with other projection dimensions is not included in the figure for clarity. We see that GSVGD1 remains robust with the sample size and achieves accurate estimation. In contrast, SVGD severely underestimates the variance in high dimensions regardless of the sample size, and S-SVGD shows slight overestimation.

\begin{figure}[t!]
\centering
\begin{minipage}{.45\textwidth}
    \centering
    \includegraphics[width=0.9\textwidth]{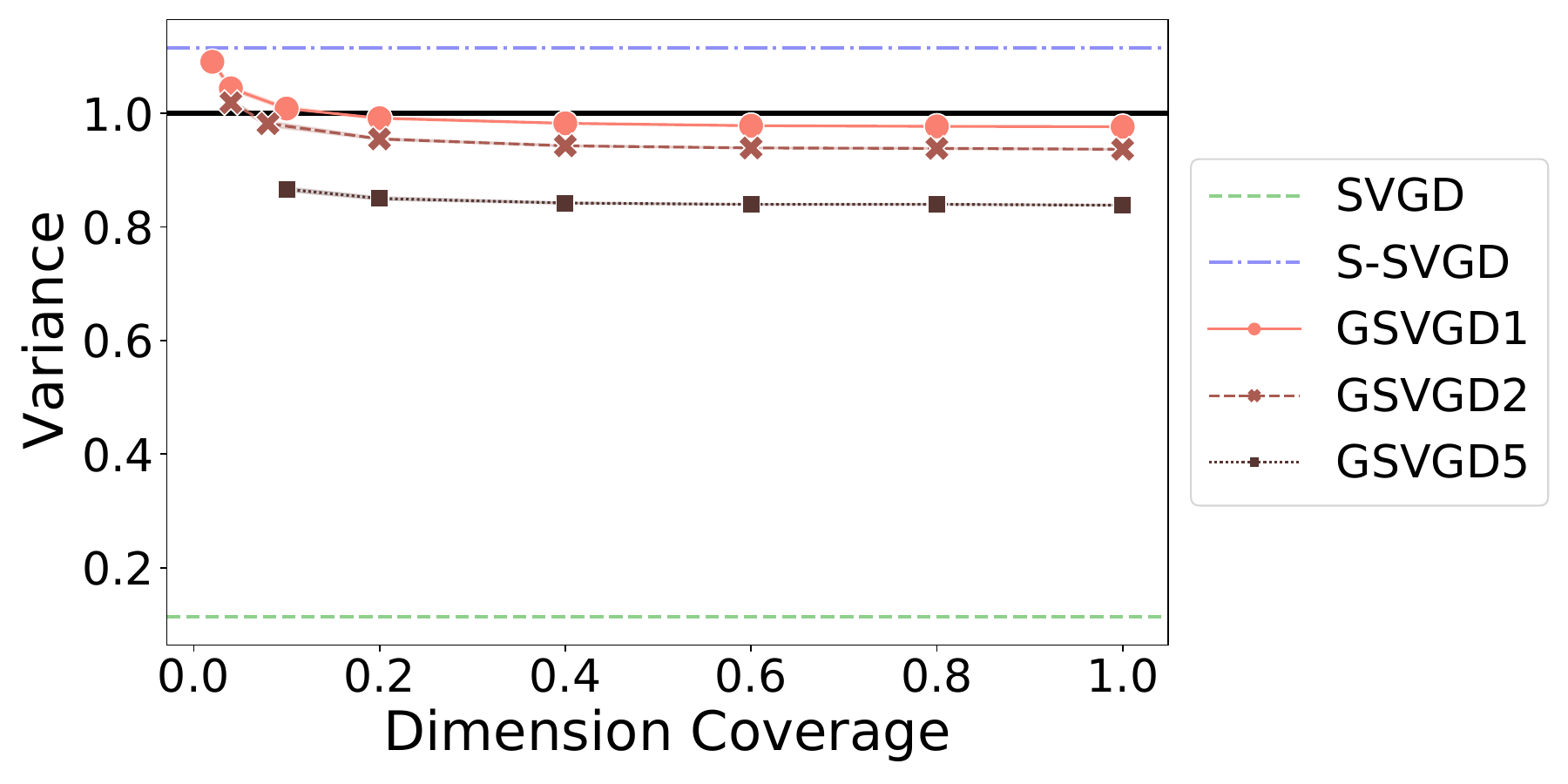}
    \caption{Ablation study on the number of projectors $M$. The target is a 50-dimensional multivariate Gaussian. Black solid line marks the true value.}
    \label{fig: gaussian ablation M}
\end{minipage}
\hfill
\begin{minipage}{.45\textwidth}
    \vspace{-2.2em}
    \centering
    \includegraphics[width=0.9\textwidth]{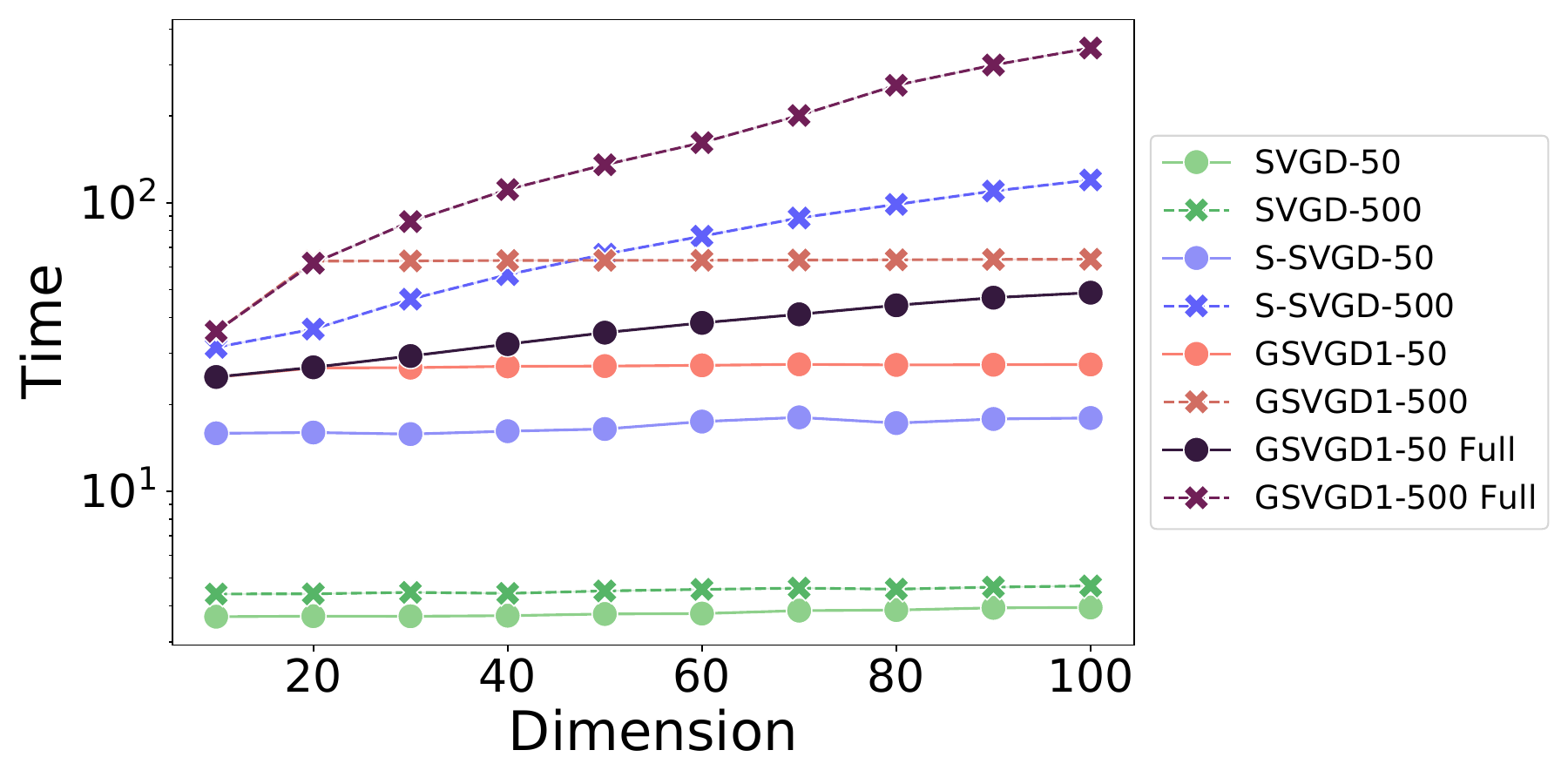}
    \caption{Run time (in seconds) against the dimension of the multivariate Gaussian target.}
    \label{fig: gaussian time}
\end{minipage}
\end{figure}

\paragraph{Number of Projectors}
In Figure~\ref{fig: gaussian ablation M}, we perform an ablation study of the performance of GSVGD with respect to the number of projectors $M$ used in Algorithm~\ref{alg:gsvgd}. The target is again the 50-dimensional multivariate Gaussian distribution. The estimate of the dimension-averaged marginal variance is plotted against the \emph{dimension covergae}, which is defined as $(M*m) / d$; that is, it is the total ranks of the $M$ projectors as a proportion of dimension. For a given $m$, we choose $M = m, 10, 20, \ldots, 50$. We also include the result for $M = 2$ and $5$ when $m=1$, and $M = 2$ when $m = 2$.

We observe that the variance estimate of GSVGD improves as $M$ increases until the dimension coverage reaches roughly 0.2, beyond which the performance stabilises. This suggests that using more projectors does not necessarily improves the estimation quality of GSVGD. It also justifies our choice of $M = \min(20, \lfloor d/m \rfloor)$ (corresponding to a dimension coverage of $0.4$ for $m=1$, $0.8$ for $m=2$ and $1.0$ for $m=5$). We remark that the estimated variance of GSVGD1 and GSVGD2 are consistently better than S-SVGD and SVGD for all choices of $M$.

\paragraph{Time Complexities}
We compare the run time of GSVGD against its competitors in Figure~\ref{fig: gaussian time}. In GSVGD, we used $M = \min(20, \lfloor d/m \rfloor)$ projectors of rank $m=1$ (GSVGD1-50, GSVGD1-500, where GSVGD1-50 means 50 particles were used). We have capped $M$ to be at most 20 instead of setting it to the dimensionality $d$ because this was the default setting in all experiments; however, since S-SVGD uses all $d$ slices, we also include the run time of GSVGD1 when using $M=d$ projectors (GSVGD1-50 Full, GSVGD1-500 Full) for a fair comparison. 

We see that the run time of GSVGD1-50 Full and GSVGD1-500 Full grows an order-of-magnitude faster than S-SVGD even though they use the same projection dimension and number of projectors/slices. This is due to the extra matrix computation in the particle update of GSVGD, which in turn arises from the fact that the projectors for the score functions are allowed to vary in GSVGD but not in S-SVGD. In particular, the time complexity for a fixed sample size is $\calO(d^3)$ for GSVGD1-50 Full and GSVGD1-500 Full, in contrast to $\calO(d^2)$ for S-SVGD. However, capping $M$ to be at most $20$ can greatly reduce the run time of GSVGD in high dimensions. In fact, with such a choice of $M$ the computational complexity reduces to $\calO(d^2)$, as can be seen from the curves GSVGD1-50 and GSVGD1-500.

Finally, the computational burden of the projector update in GSVGD1-50 Full and GSVGD1-500 Full ($\calO(d^2)$) is less than that of the particle update. Hence, the main computational bottleneck of GSVGD arises from the particle update (\ref{eq: gsvgd update}).

\thispagestyle{empty}

\end{document}